\documentclass[twoside]{article}
\usepackage[preprint]{aistats2027}
\usepackage[round]{natbib}

\usepackage{amsmath,amssymb,amsthm,mathtools,mathrsfs,bbm}
\usepackage{centernot,wrapfig,xspace}
\usepackage{booktabs,enumitem,graphicx,float}
\usepackage{subcaption}
\graphicspath{{figures/}}
\usepackage{xcolor}
\definecolor{modeblue}{HTML}{1D4ED8}
\definecolor{modered}{HTML}{B91C1C}
\usepackage{algorithm}
\usepackage[noend]{algorithmic}

\usepackage{hyperref,url}
\hypersetup{hidelinks}
\usepackage{cleveref}
\newtheorem{theorem}{Theorem}
\newtheorem{proposition}[theorem]{Proposition}

\newtheorem{corollary}[theorem]{Corollary}
\theoremstyle{definition}

\newtheorem{assumption}{Assumption}
\newcommand{\cC}{\mathcal C}
\newcommand{\cK}{\mathcal K}
\newcommand{\cR}{\mathcal R}
\newcommand{\cA}{\mathcal A}
\newcommand{\bE}{\mathbb E}
\newcommand{\bP}{\mathbb P}
\newcommand{\bR}{\mathbb R}
\newcommand{\one}{\mathbbm 1}
\renewcommand{\d}{\mathrm{d}}
\DeclareMathOperator*{\argmin}{arg\,min}

\newcommand{\casa}{\texttt{CASA}}

\newcommand*{\eg}{\emph{e.g.}\xspace}

\begin{document}
\twocolumn[
\aistatstitle{Sample-Efficient Generative Conformal Prediction}
\aistatsauthor{Minxing Zheng \And Shixiang Zhu}
\aistatsaddress{Carnegie Mellon University \And Carnegie Mellon University}
]

\begin{abstract}
Generative conformal prediction builds uncertainty sets from samples of a conditional generator, which are efficient only when the samples represent the response distribution well. This can require many samples, each of which can be costly, as in large diffusion models and scientific simulators, so the sampling budget must be used efficiently. Existing methods draw the same number of samples at every input, wasting samples where the response distribution is simple and undersampling where it is complex, which inflates sets and leaves those inputs under-covered. We propose \casa{} (Conformal Adaptive Sample Allocation), which characterizes the marginal value of an additional sample and allocates samples across inputs to minimize the expected set size subject to marginal coverage and an average sampling budget. Theoretical analysis shows that adaptive allocation yields smaller sets than a fixed count at the same budget: a missed mode forces a radius that spans the gap between modes, and even oracle radius cannot compensate for it. On synthetic and real tasks, \casa{} produces substantially smaller sets at the same budget, often improves conditional coverage, and complements existing radius-adaptive methods.
\end{abstract}

\section{Introduction}

Conformal prediction provides distribution-free uncertainty sets for arbitrary predictive models under exchangeability \citep{vovk2005algorithmic,angelopoulos2023conformal}. This model-agnostic guarantee is especially useful for black-box generators, including diffusion models, large language models, and scientific simulators. Rather than conformalizing a single point prediction, generative conformal methods construct sets from multiple model samples, allowing the set to reflect complex or multimodal predictive distributions \citep{wang2023pcp}. Probabilistic conformal prediction (PCP), for example, places calibrated balls around generated samples; unlike a connected region, their union can cover separated modes without filling the low-density regions between them.

This efficiency, however, depends on drawing enough samples to represent the predictive distribution. In many applications, each draw is itself costly: diffusion models require repeated network evaluations, language models decode outputs sequentially, and scientific simulators may require large-scale computing \citep{ho2020denoising,brown2020language}. Climate modeling provides a particularly stark example: a high-resolution Community Earth System Model (CESM) run used $23{,}404$ processor cores, advanced by only about two simulated years per day, and consumed approximately $250{,}000$ processor-hours for each simulated year \citep{small2014synoptic}.
Thus, a single $100$-year ensemble member required roughly $50$ days of wall-clock time on more than $23{,}000$ cores---about $25$ million processor-hours.
Even at coarser resolution, the CESM Large Ensemble project required more than 17 million core-hours across its simulations \citep{kay2015cesm}. When each sample is a decades- or centuries-long climate trajectory, drawing a large, fixed number of samples for every initial condition or forcing scenario is computationally prohibitive.

Applying conformal prediction to such models therefore requires using a limited sampling budget efficiently. Existing methods treat the sample count as a fixed hyperparameter and use the same number of draws for every input, despite substantial heterogeneity in their value. A few draws may adequately represent a concentrated distribution, whereas the same count may miss an entire mode of a more complex distribution. Uniform sampling consequently wastes computation on simple inputs while undersampling the difficult inputs that drive prediction-set size and conditional-coverage imbalance. This raises our central question: \emph{how should a limited sampling budget be allocated across inputs to construct small conformal sets while mitigating conditional-coverage imbalance?}

To this end, we propose \casa{} (Conformal Adaptive Sample Allocation), which learns an input-dependent sample count under an average sampling budget. Its guiding principle is simple: allocate additional draws where they are likely to recover response mass missed by the current samples, and stop sampling where further draws mainly add set volume. In the example of Figure~\ref{fig:motivation}, this means moving samples from a concentrated input, where two draws suffice, to a multimodal input, where additional draws can reveal a missing mode. Better mode representation can sharply reduce the calibrated radius, shrinking prediction sets even at inputs that receive fewer samples.

\begin{figure}[!t]
\centering
\noindent\makebox[0.5\columnwidth][c]{\color{modeblue}\small $Y\mid X=x_1$}%
\makebox[0.5\columnwidth][c]{\color{modered}\small $Y\mid X=x_2$}\par\vspace{2pt}
\begin{subfigure}{\columnwidth}
\centering
\includegraphics[width=\linewidth]{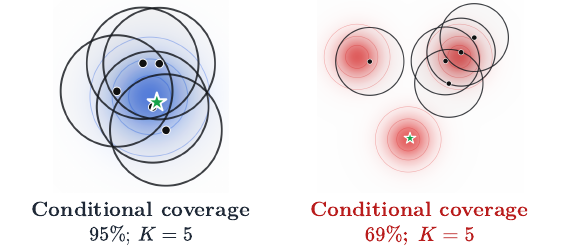}
\caption{Conventional approach (mean set size $6.69$).}
\label{fig:motivation-fixed}
\end{subfigure}
\vspace{0.3em}
\begin{subfigure}{\columnwidth}
\centering
\includegraphics[width=\linewidth]{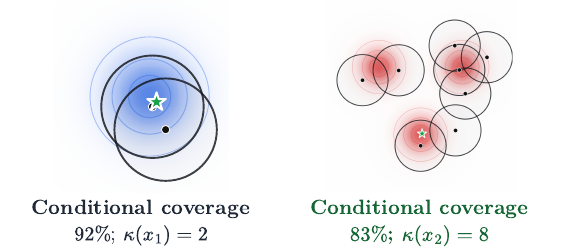}
\caption{Sample-efficient approach (mean set size $6.24$).}
\label{fig:motivation-adaptive}
\end{subfigure}
\caption{Motivation for adaptive sampling.
Moving three draws from $x_1$ to $x_2$ improves mode representation and balances conditional coverage at the same average budget.}
\label{fig:motivation}
\end{figure}

The key insight is that the marginal coverage benefit and set-size cost of an additional draw are both characterized by the same inclusion probability, which can be estimated unbiasedly from generator samples. We use these estimates to formulate sample allocation as a constrained optimization problem that minimizes expected set size subject to coverage and sampling-budget constraints. Its two-multiplier dual admits a near-linear-time solution and a natural shadow-price interpretation: an input receives another draw when its marginal coverage benefit offsets its volume and sampling costs. Independent conformal calibration preserves finite-sample marginal coverage, and the allocation can be combined with radius-adaptive methods such as RCP and CP$^2$ \citep{plassier2025rcp,plassier2025cp2}.

Our analysis explains why adapting the sample count can reduce prediction-set size. When too few samples are drawn from a multimodal distribution, an entire mode may be missed, forcing calibration to use a large radius that spans the gap between modes. Allocating more samples to such inputs makes it more likely that every relevant mode is represented, allowing a much smaller radius. This benefit remains even when the radius is optimally adapted to each input, showing that sample allocation provides gains that radius adaptation alone cannot achieve.

Empirically, across various synthetic benchmarks, \casa{} reduces mean set size by approximately $12\%$ relative to fixed-count PCP. Both maintain marginal coverage and reduce conditional-coverage error. On Porto taxi trajectories, \casa{} allocates more draws near junctions and fewer along straight roads, reducing mean set area by approximately $37\%$ with five draws per input on average. These gains persist at matched sampling costs and extend to several existing generative conformal methods.

The main contributions are summarized as follows: ($i$) We formulate adaptive sampling for generative conformal prediction as a resource-allocation problem. ($ii$) We derive an exact, near-linear-time algorithm and characterize efficiency gains beyond radius adaptation. ($iii$) Across synthetic and real tasks, we show that \casa{} reduces set size at comparable sampling costs, improves conditional-coverage balance, and complements existing generative conformal methods.

\paragraph{Related work.}
Conformal prediction provides uncertainty sets with finite-sample, distribution-free coverage under exchangeability \citep{vovk2005algorithmic,shafer2008tutorial}. Split conformal methods enable efficient calibration of modern predictors \citep{lei2018distribution,romano2019conformalized}, with extensions addressing distribution shift \citep{tibshirani2019conformal,barber2023conformal} and general risk control \citep{bates2021distribution,angelopoulos2024conformal,zhou2026calibrating}. Its scope now covers time series \citep{gibbs2021adaptive,xu2021conformal,zaffran2022adaptive,xu2024multidimensional}, functional outputs and operator models \citep{lei2015functional,diquigiovanni2022functional,harris2025locally}, and noisy, counterfactual, or latent targets \citep{einbinder2024label,alaa2023meta,javanmardi2023partial,zheng2026latent}. Other work moves beyond sets to calibrated predictive distributions, such as conformal flows \citep{harris2026flow}. Exact conditional coverage is unattainable without further assumptions \citep{barber2021limits}, which motivates locally adaptive scores \citep{lei2018distribution,romano2019conformalized} and group-conditional relaxations \citep{gibbs2025conformal}.

Conformal sets can also be built from estimates of the conditional distribution, such as CDFs \citep{chernozhukov2021distributional}, densities \citep{izbicki2022cd}, or generated samples. Among sample-based methods, PCP forms a union of balls around generated samples \citep{wang2023pcp}. Later work refines its score to improve conditional coverage, either by rescaling it with an input-dependent radius, as in RCP and CP$^2$ \citep{plassier2025rcp,plassier2025cp2}, or by mapping it through an estimated conditional distribution function, as in C-PCP \citep{dheur2025unified}. Others adapt the set's geometry to the structure of the samples, through clustering in CP4Gen \citep{yang2026cp4gen} or ranked generated regions \citep{zheng2024generative}. All of these methods draw the same number of samples at every input. For discrete outputs such as text, related methods adapt the number of samples to the input \citep{quach2024clm,shahrokhi2025gps,noorani2025cpq}; we address continuous responses, where sets are geometric and an additional sample changes both their coverage and their size.

The adaptive allocation of a sampling budget is a classical problem in statistics. Neyman allocation distributes samples across strata in proportion to their variability \citep{neyman1934two}, adaptive importance sampling concentrates draws where they most reduce estimator variance \citep{bugallo2017adaptive}, and best-arm identification allocates a fixed budget of pulls across arms \citep{audibert2010best}. In machine learning, adaptive computation methods assign more computation to harder inputs, from learned halting in recurrent networks \citep{graves2016adaptive} to test-time compute in language models, where repeated sampling improves accuracy and its optimal amount depends on problem difficulty \citep{wang2023self,snell2025scaling} and can be allocated per query \citep{damani2025hard,zuo2026strategic}. These methods target estimator variance, accuracy, or task reward, whereas \casa{} allocates generator samples to shrink calibrated uncertainty sets under a coverage constraint.

\section{Setup and Preliminaries}
\label{sec:setup}\label{sec:prelim}

Consider $(X,Y)\sim\mathcal{P}$, where $X\in\mathcal{X}$ and $Y\in\mathcal{Y}\subseteq\mathbb{R}^d$, and let $\mathcal{P}_x$ denote the conditional distribution of $Y$ given $X=x$. Let $\widehat f$ be a fixed predictive model, and let $S:\mathcal{X}\times\mathcal{Y}\to\mathbb{R}$ be a nonconformity score measuring the discrepancy between a candidate response $y$ and the model prediction at $x$, \eg, $S(x,y)=\|y-\widehat f(x)\|_2$.

Given a calibration set $\mathcal{D}_n=\{(X_i,Y_i)\}_{i=1}^n$, split conformal prediction computes $S_i=S(X_i,Y_i)$ and, for a target miscoverage level $\alpha\in(0,1)$, constructs
\begin{equation}
\widehat{\mathcal{C}}_\alpha(x)
=
\left\{
y\in\mathcal{Y}: S(x,y)\le \widehat q
\right\},
\label{eq:split}
\end{equation}
where
\[
\widehat q
=
Q\left(
S_{1:n};
\frac{\left\lceil (n+1)(1-\alpha)\right\rceil}{n}
\right),
\]
and $Q(S_{1:n};\tau)$ denotes the empirical $\tau$-quantile of $S_1,\ldots,S_n$, with the convention $Q(S_{1:n};\tau)=\infty$ for $\tau>1$. If the calibration observations and the test pair $(X_{\mathrm{test}},Y_{\mathrm{test}})$ are exchangeable, then
\begin{equation}
\mathbb{P}\left\{
Y_{\mathrm{test}}
\in
\widehat{\mathcal{C}}_\alpha(X_{\mathrm{test}})
\right\}
\ge 1-\alpha,
\label{eq:coverage}
\end{equation}
where the probability is taken jointly over the calibration data and the test observation.

The guarantee in \eqref{eq:coverage} holds for any choice of nonconformity score, while the score determines the informativeness of the resulting prediction sets. We focus on two properties: First is the \emph{efficiency}, measured by the expected set size $v \coloneqq \mathbb{E}|\widehat{\mathcal{C}}_\alpha(X)|$, where $|\cdot|$ denotes an appropriate notion of size, such as Lebesgue measure for continuous outcomes or cardinality for discrete outcomes; and second is the \emph{conditional coverage}, given by $c(x)\coloneqq\mathbb{P}\{Y\in\widehat{\mathcal{C}}_\alpha(x)|X=x\}$. The former rules out trivially valid but uninformative sets such as $\widehat{\mathcal{C}}_\alpha(x)=\mathcal{Y}$, while the latter captures heterogeneity in coverage across covariate values or subpopulations that is not controlled by the marginal guarantee in \eqref{eq:coverage} \citep{barber2021limits}.

Recently, \cite{wang2023pcp} has introduced probabilistic conformal prediction (PCP), which can substantially improve the efficiency of prediction sets. Specifically, suppose that $\widehat f$ is a conditional generative model from which, for any input $x$, we can draw independent samples $\widehat Y_1,\ldots,\widehat Y_K\overset{\mathrm{iid}}{\sim}\widehat f(x)$. The PCP measures the nonconformity of a candidate response $y$ by its distance to the nearest generated sample,
\begin{equation}
S(x,y)=
\frac{\min_{j\le K}\|y-\widehat Y_j\|_2}{r(x)},
\label{eq:gcp}
\end{equation}
leading to the prediction set
\begin{equation}
\widehat{\mathcal{C}}_\alpha(x)=
\bigcup_{j=1}^K
\mathcal{B}\!\left(\widehat Y_j,\widehat q\,r(x)\right),
\label{eq:pcpset}
\end{equation}
where $\mathcal{B}(y,r)$ denotes the Euclidean ball of radius $r$ centered at $y$, and $r(x)>0$ is an input-dependent scaling function that adapts the conformal radius.
For example, PCP sets $r(x)\equiv 1$ \citep{wang2023pcp}, whereas subsequent adaptive-radius approaches learn $r(x)$ from the input \citep{plassier2025rcp,plassier2025cp2}.

However, the efficiency of these constructions depends on having a sufficiently large number of samples $K$ at each input, which can become a computational bottleneck when sample generation is expensive. When $K$ is small, three limitations arise. ($i$) Increasing the conformal radius cannot compensate for missing support: if $\mathcal{P}_x$ is multimodal and all $K$ samples miss the mode containing the realized response, then $S(x,Y)$ is governed by the separation between modes, and any ball large enough to cover the response must also span the intervening low-density region, regardless of $r(x)$. ($ii$) Such misses propagate through calibration. Because $\widehat q$ is a global quantile of scores pooled across inputs, frequent misses inflate $\widehat q$ and enlarge prediction sets broadly, whereas rare misses can leave specific inputs under-covered, with $c(x)<1-\alpha$, compensated by over-coverage elsewhere. ($iii$) A uniform sampling budget can be inefficient: when $\mathcal{P}_x$ is highly concentrated, a few samples may already suffice to make $S(x,Y)$ small, so additional samples yield limited benefit relative to their computational cost. This suggests allocating samples adaptively across inputs, using fewer samples where $\mathcal{P}_x$ is concentrated and more where it is complex or multimodal. We formalize this allocation problem next.

\section{Methodology}
\label{sec:method}

\begin{figure*}[t]
\centering
\begin{subfigure}[t]{0.245\textwidth}
\centering
\includegraphics[width=\linewidth]{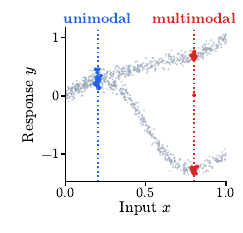}
\caption{Distribution $Y\mid X$}
\label{fig:mechanism-data}
\end{subfigure}\hfill
\begin{subfigure}[t]{0.245\textwidth}
\centering
\includegraphics[width=\linewidth]{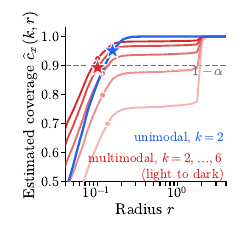}
\caption{Count--radius tradeoff}
\label{fig:mechanism-coverage}
\end{subfigure}\hfill
\begin{subfigure}[t]{0.245\textwidth}
\centering
\includegraphics[width=\linewidth]{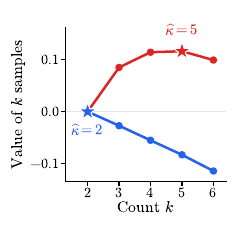}
\caption{Marginal value of count}
\label{fig:mechanism-value}
\end{subfigure}\hfill
\begin{subfigure}[t]{0.245\textwidth}
\centering
\includegraphics[width=\linewidth]{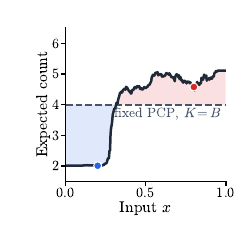}
\caption{Resulting allocation}
\label{fig:mechanism-alloc}
\end{subfigure}
\caption{Mechanism of \casa{} with average budget $B=4$ on a conditional distribution whose modes separate as $x$ grows. (b)~Dots mark the radius minimizing the priced cost $g$ in Theorem~\ref{thm:dual} for each count, and stars mark the selected count--radius pair. (c)~The value of $k$ samples is $-\min_r g_k(x;r)$ relative to $k=2$; each input takes the count of highest value. (d)~\casa{} moves samples from unimodal to multimodal inputs while keeping the same average budget as fixed PCP.}
\label{fig:mechanism}
\end{figure*}

\subsection{Resource Allocation Representation}
\label{sec:formulation}
\label{sec:procedure}
\label{sec:tradeoff}
\label{sec:allocation}
\label{sec:calibration}

The goal is to design an input-dependent sampling rule that allocates the sampling budget adaptively across inputs. Let $\kappa(x)\in\mathbb{N}$ denote the number of samples allocated to input $x$. We formulate the design of $\kappa$ as a resource-allocation problem that minimizes the expected prediction-set size subject to marginal coverage and an average sampling-budget constraint. Formally,
\begin{equation}
\begin{aligned}
\min_{\kappa(\cdot)}\quad
& \bE\;\bigl|\cC_{\kappa}(X)\bigr|
\coloneqq \bE\;\Bigl|\bigcup\nolimits_{j=1}^{\kappa(X)}
\mathcal{B}\!\left(\widehat Y_j\right)\Bigr| \\
\text{s.t.}\quad
& \bP\!\left\{Y\in\cC_{\kappa}(X)\right\}\ge 1-\alpha, \\
& \bE[\kappa(X)]\le B,
\end{aligned}
\label{eq:population}
\end{equation}
where $\widehat Y_1,\widehat Y_2,\ldots\overset{\mathrm{iid}}{\sim}\widehat f(X)$, and $B$ denotes the average sampling budget. Here, $\mathcal{B}(y)$ denotes a ball centered at $y$, whose radius may be either a global parameter $r$ or an input-dependent function $r(x)$. In principle, the radius and the sampling rule $\kappa$ can be jointly optimized. For clarity, we focus on optimizing sample allocation $\kappa$ and treat the radius specification separately.

For a given input $x$, define
\[
m(x;y)
=
\bP_{\widehat Y\sim\widehat f(x)}
\!\left\{\widehat Y\notin\mathcal B(y)\right\},
\]
the probability that a single draw from the conditional generative model does not fall within the ball $\mathcal B(y)$.
With $k$ independent draws, the probability that at least one draw falls within $\mathcal B(y)$, termed as \emph{inclusion probability}, is
\begin{equation}
h_k(x;y)=1-m(x;y)^k.
\label{eq:inclusion}
\end{equation}
Accordingly, define the conditional coverage probability and expected prediction-set size at input $x$ as
\[
c_k(x)
=
\int h_k(x;y)\,\d\mathcal P_x(y),\quad
v_k(x)
=
\int h_k(x;y)\,\d y,
\]
where the second integral is taken with respect to the underlying reference measure on the response space.
Thus, the marginal coverage and expected set size in \eqref{eq:population} can be written as $\bE[c_{\kappa(X)}(X)]$ and $\bE[v_{\kappa(X)}(X)]$.

A key observation is that the marginal increase in inclusion probability from one additional draw is
$
\Delta h_k=h_{k+1}-h_k.
$
Integrating this increment yields the corresponding marginal gain in coverage and marginal increase in expected set size. An additional draw is therefore valuable when it captures substantial response mass that is likely to be missed by the existing draws, for example, an underrepresented mode. Both marginal effects decrease with $k$, reflecting diminishing returns to additional sampling (Appendix~\ref{app:proofs}).

\begin{algorithm}[!t]
\caption{\casa{}}
\label{alg:allocation}
\begin{algorithmic}[1]
\makeatletter
\renewenvironment{ALC@g}{\begin{list}{\ALC@lno}{\itemsep\z@\itemindent\z@\listparindent\z@\rightmargin\z@\topsep\z@\partopsep\z@\parskip\z@\parsep\z@\leftmargin\algorithmicindent\addtolength{\ALC@tlm}{\leftmargin}}}{\end{list}}
\REQUIRE \parshape 2 \@totalleftmargin \linewidth \dimexpr\@totalleftmargin-\ALC@tlm\relax \dimexpr\linewidth+\ALC@tlm\relax model $\widehat f$, data $\mathcal D_m$, $\mathcal D_n$, candidate counts $\cK$, budget $B$, level $\alpha$, pool size $N$, radius $r(\cdot)$ (optional)
\makeatother
\FOR[Step 1]{$(X_i,Y_i)\in\mathcal D_m$}
\STATE Draw $\widehat Y_{i,1},\ldots,\widehat Y_{i,N}\sim\widehat f(X_i)$
\STATE Compute $\widehat c_k(X_i)$ and $\widehat v_k(X_i)$, $k\in\cK$, by \eqref{eq:perinput}
\ENDFOR
\STATE Regress $\widehat c_k(X_i)$, $\widehat v_k(X_i)$ on $X_i$ to obtain $\widehat c_k$, $\widehat v_k$
\STATE Solve dual optimization in \eqref{eq:dualLP} \COMMENT{Step 2}
\STATE Set $\widehat\kappa(x)\in\argmin_{k\in\cK}g_k(x)$
\FOR[Step 3]{$(X_i,Y_i)\in\mathcal D_n$}
\STATE Draw $\widehat Y_{i,1},\ldots,\widehat Y_{i,\widehat\kappa(X_i)}\sim\widehat f(X_i)$
\STATE Compute $S_i$ by \eqref{eq:gcp} with $K=\widehat\kappa(X_i)$
\ENDFOR
\STATE Compute $\widehat q$ by \eqref{eq:split}
\ENSURE $\widehat\cC_\alpha(X_{\mathrm{test}})=\bigcup_{j=1}^{\widehat\kappa(X_{\mathrm{test}})}\mathcal B\bigl(\widehat Y_j,\widehat q\,r(X_{\mathrm{test}})\bigr)$
\end{algorithmic}
\end{algorithm}

\subsection{Proposed Algorithm}

We develop \casa{} (Algorithm~\ref{alg:allocation}) that solves problem~\eqref{eq:population} in three steps: ($i$) Estimate $c_k$ and $v_k$ for each candidate count $k$ from generated samples; ($ii$) Plug the empirical estimates into \eqref{eq:population} and solve its dual form; and ($iii$) Calibrate the resulting uncertainty sets to guarantee finite-sample marginal coverage.

\paragraph{Step 1: Estimate coverage and size.} Given an independent dataset $\mathcal D_m$, for each $(X_i,Y_i)\in\mathcal D_m$, we first draw $N\ge k$ predictions
\[
\widehat Y_{i,1},\ldots,\widehat Y_{i,N}\overset{\mathrm{iid}}{\sim}\widehat f(X_i).
\]
These $N$ predictions are generated only once during the offline estimation stage and can be reused to evaluate all candidate values of $k$, so this step incurs no additional sampling cost at testing phase.

Next we consider drawing $k$ samples uniformly without replacement from these $N$ predictions. Let
\[
H(X_i;y)
=
\sum_{j=1}^N
\one\left\{\|\widehat Y_{i,j}-y\|_2\le r\right\}
\]
denote the number of predictions whose radius-$r$ balls contain $y$. Since the number of such predictions selected among the $k$ draws follows a hypergeometric distribution, the inclusion probability in \eqref{eq:inclusion} admits the following unbiased estimator
\begin{equation}
\widehat h_k(X_i; y)
=
1-
\binom{N-H(X_i;y)}{k}
\Big/
\binom{N}{k}.
\label{eq:rb}
\end{equation}
See Proposition~\ref{prop:rb} in Appendix~\ref{app:proofs} for details.

Lastly, we estimate the conditional coverage and set size at each $X_i$ by
\begin{equation}
\widehat c_k(X_i)
=
\widehat h_k(X_i;Y_i),
\quad
\widehat v_k(X_i)
=
\int \widehat h_k(X_i;y)\,\d y.
\label{eq:perinput}
\end{equation}
To predict these two quantities at a new input $X_{\mathrm{test}}$, we fit regression models for $c_k(x)$ and $v_k(x)$ using $\{(X_i,\widehat c_k(X_i),\widehat v_k(X_i))\}_{i=1}^m$. The fitted models can then be evaluated at $X_{\mathrm{test}}$ to obtain the corresponding coverage and size estimates. Appendix~\ref{app:implementation} provides implementation details and extensions to more radius-adaptive constructions.

\paragraph{Step 2: Reformulation.}
Replacing $c_k$ and $v_k$ in \eqref{eq:population} by $\widehat c_k$ and $\widehat v_k$ and expectations by averages over $\mathcal D_m$, and restricting the counts to a finite candidate set $\cK\subseteq\mathbb N$, gives its empirical representation, which can be expressed as a linear program (Appendix~\ref{app:dual}). The empirical problem has $m|\cK|$ variables, and interior-point methods solve it in $O\bigl((m|\cK|)^{3.5}\bigr)$ in the worst case \citep{karmarkar1984new}, so its cost grows polynomially and quickly. Empirically, with $m=5{,}000$ and $|\cK|=10$, a standard solver takes about $11$ seconds on a single CPU core, and about $43$ seconds at $m=10{,}000$. This motivates a more efficient solution based on the following exact dual of the empirical problem.
\begin{theorem}[Dual reformulation]
\label{thm:dual}
Suppose the empirical problem, over randomized rules $\kappa$, is feasible. Then it is equivalent to
\begin{equation}
\begin{aligned}
\max_{\mu,\lambda\ge0,\;u\in\bR^m}\;&\mu(1-\alpha)-\lambda B+\frac1m\sum_{i=1}^m u_i\\
\mathrm{s.t.}\;&u_i\le g_k(X_i),\quad i=1,\ldots,m,\ k\in\cK,
\end{aligned}
\label{eq:dualLP}
\end{equation}
where $g_k(x)=\widehat v_k(x)-\mu\,\widehat c_k(x)+\lambda k$, $\mu,\lambda\ge0$ are the dual variables for the constraints, and $\cK\subseteq\mathbb N$ is a finite set of candidate counts, \eg, $\{1,\ldots,2B\}$.
\end{theorem}
The proof is deferred to Appendix~\ref{app:dual}.

The dual is cheaper to solve because it decouples the inputs. The primal couples all $m$ inputs through the coverage and budget constraints, whereas for fixed dual variables $(\mu,\lambda)$ the dual problem separates into $m$ independent subproblems, $\min_{k\in\cK}g_k(X_i)$. Solving the dual therefore reduces to a search over two dual variables. Theoretically, for a fixed $\mu$, the dual is solved in $O\bigl(m|\cK|\log(m|\cK|)\bigr)$, which is near-linear, compared with $O\bigl((m|\cK|)^{3.5}\bigr)$ for the primal. Empirically, in the same setting, the dual takes about $0.17$ seconds at $m=5{,}000$ and $0.4$ seconds at $m=10{,}000$, compared with $11$ and $43$ seconds for the primal.

\paragraph{Step 3: Calibration.}
Finally, we calibrate the learned allocation rule $\widehat\kappa$ on $\mathcal D_n$. For each $(X_i,Y_i)\in\mathcal D_n$, we draw $\widehat\kappa(X_i)$ samples from the generator and compute the score in \eqref{eq:gcp} using $ K = \widehat\kappa(X_i)$. We then obtain the conformal quantile $\widehat q$ from \eqref{eq:split}. At a test input $x$, we draw $\widehat\kappa(x)$ samples and return the set \eqref{eq:pcpset}.
\begin{theorem}[Finite-sample marginal coverage]
\label{thm:validity}
Suppose that the generator $\widehat f$ and the learned $\widehat\kappa$ are independent of $\mathcal D_n$ and the test pair, and that $D_n$ and test pair are exchangeable. Then
\[
\bP\{Y_{\mathrm{test}}\in
\widehat\cC_\alpha(X_{\mathrm{test}})\}
\ge 1-\alpha.
\]
\end{theorem}
The proof is deferred to Appendix~\ref{app:proofs}.

\subsection{Theoretical Analysis}
\label{sec:theory}

\label{sec:conditional}

We next characterize when adaptive allocation improves over uniform allocation under well-separated modes (Assumption~\ref{ass:modes}). This setting isolates the effect of missed modes and makes the marginal value of each additional sample explicit; Appendix~\ref{app:dual} gives a distribution-free characterization of the gain. All proofs are deferred to Appendix~\ref{app:theory}.

\begin{assumption}[Separated modes]
\label{ass:modes}
There exist $\epsilon>0$ and $R>1$ such that, for $x\in\mathcal X$, both $\mathcal P_x$ and $\widehat f(x)$ are supported on $\bigcup_{j=1}^{J(x)}\mathcal B(z_j(x),\epsilon)$, whose centers satisfy
\[
\min_{i\ne j}\|z_i(x)-z_j(x)\|_2\ge2(R+1)\epsilon.
\]
\end{assumption}
Assumption~\ref{ass:modes} states that $\mathcal P_x$ has $J(x)$ modes, each covered by a ball of radius $\epsilon$, with mode width at most $2\epsilon$. Each mode is separated from the others by gaps of at least $2R\epsilon$. Here $R$ can be understood as the ratio of the gap $2R\epsilon$ to the mode width $2\epsilon$.

We do not require $\widehat f(x)$ to be consistent with $\mathcal P_x$: the generator only needs to be supported on the same balls, and may weight them differently or even miss some modes of $\mathcal P_x$. Let
\[
p_j(x)=\mathcal P_x\{\mathcal B(z_j(x),\epsilon)\},\quad \widehat p_j(x)=\widehat f(x)\{\mathcal B(z_j(x),\epsilon)\}
\]
denote the probabilities of the $j$-th ball under $\mathcal P_x$ and $\widehat f(x)$, respectively. Then for any radius $r\in[2\epsilon,2R\epsilon)$, the conditional coverage does not depend on $r$ (Appendix~\ref{app:theory}):
\begin{equation}
c_k(x)=1-\sum_{j=1}^{J(x)}p_j(x)\{1-\widehat p_j(x)\}^k.
\label{eq:modecov}
\end{equation}

\begin{theorem}[Efficiency gain at equal budget]
\label{thm:threshold}
Under Assumption~\ref{ass:modes}, let every count rule use the smallest constant radius attaining marginal coverage $1-\alpha$. For any budget $B\in\cK$ satisfying
\[
 \bE[c_B(X)]<1-\alpha\le\max_{\bE[\kappa(X)]\le B}\bE[c_{\kappa(X)}(X)],
\]
the expected set sizes $V_A$ and $V_F$ of the adaptive and uniform allocation, respectively, satisfy
\[
 \frac{V_A}{V_F}\le\frac{B}{R^d}.
\]
\end{theorem}

Theorem~\ref{thm:threshold} shows that, at the same budget, moving samples from easy to hard inputs can shrink prediction sets by a factor of at least $R^d/B$. Intuitively, a missed mode can be covered in two ways: by drawing more samples, or by enlarging the radius until it bridges the gap. Uniform allocation is forced into the latter at hard inputs, and since the radius is shared, every input pays for it. Adaptive allocation instead spends the samples saved at easy inputs, \eg, unimodal ones covered by a single sample, on the modes of hard inputs, so every input keeps a ball of the mode's size. The gain is therefore largest when modes are well separated and inputs differ in how many samples they need.

Adaptive allocation also saves samples for conditional coverage. For a constant radius $r\in[2\epsilon,2R\epsilon)$, let $k_\alpha(x)=\min\{k\in\cK:c_k(x)\ge1-\alpha\}$ be the number of samples that input $x$ needs.

\begin{corollary}[Sampling savings for conditional coverage]
\label{cor:conditional}
Under Assumption~\ref{ass:modes}, to attain conditional coverage $1-\alpha$ at every input, adaptive allocation saves, relative to uniform allocation, the following number of samples per input:
\[
 \sup_xk_\alpha(x)-\bE[k_\alpha(X)].
\]
\end{corollary}

Corollary~\ref{cor:conditional} shows that conditional coverage costs uniform allocation the worst-case number of samples, but adaptive allocation only the average. Intuitively, each mode costs about $\log(1/\alpha)$ samples to cover, so uniform allocation spends the budget of the most multimodal input everywhere. The next result shows that even input-dependent radii cannot close this gap.
\begin{figure*}[!t]
\centering
\includegraphics[width=\textwidth]{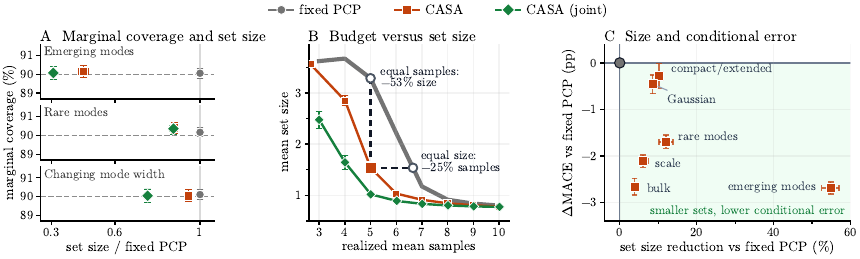}
\caption{Synthetic results on a subset of DGPs with a $90\%$ coverage target. (A)~Marginal coverage and set size relative to fixed PCP at $B=5$. (B)~Set size against realized mean samples with $B=3,\ldots,10$ on the emerging-modes DGP; open circles mark fixed PCP at equal samples and at equal size. (C)~Size reduction and MACE change of \casa{} at $B=5$. Error bars show 95\% intervals across $100$ independent runs.}
\label{fig:frontiers}
\end{figure*}
\begin{proposition}[Efficiency beyond oracle radii]
\label{prop:oracle}
Under Assumption~\ref{ass:modes}, let each input use the smallest radius attaining conditional coverage $1-\alpha$. For any budget $B\in\cK$ with $\bE[k_\alpha(X)]\le B<\sup_xk_\alpha(x)$, the expected set sizes $V_A$ and $V_F$ of the adaptive and uniform allocation, respectively, satisfy
\[
 \frac{V_A}{V_F}\le\frac{B}{\bP\{k_\alpha(X)>B\}\,R^d}.
\]
\end{proposition}

Proposition~\ref{prop:oracle} shows that a radius cannot replace a sample. Below $k_\alpha(x)$ samples, even the best radius at input $x$ must bridge a gap, so uniform allocation pays a ball of radius at least $2R\epsilon$ at every input with $k_\alpha(x)>B$, whereas adaptive allocation avoids it. Thus, radius adaptation and count allocation are complementary. Once every mode is covered, the remaining gain is of order $1/B$ (Appendix~\ref{app:efficiency}).

\section{Experiments}
\label{sec:experiments}

\textbf{Baselines and metrics.} We evaluate \casa{} on synthetic data-generating processes (DGPs) and on real-world datasets.
We compare with PCP \citep{wang2023pcp} using $K=B$ samples and one calibrated radius (\textbf{fixed PCP}), and consider two variants of \casa{}: \textbf{\casa{}} adapts only the count, and \textbf{\casa{} (joint)} adapts both the count and the radius. We also apply the count allocation of \casa{} to other generative conformal methods that adapt the radius or shape of the set: CP$^2$ \citep{plassier2025cp2}, RCP \citep{plassier2025rcp}, CP4Gen \citep{yang2026cp4gen}, and C-PCP \citep{dheur2025unified}. All methods use the same generator and the same labeled data $\mathcal D_m\cup\mathcal D_n$: \casa{} uses $\mathcal D_m$ for allocation and $\mathcal D_n$ for calibration, whereas fixed PCP calibrates on both. We report marginal coverage, mean uncertainty set size, realized mean number of samples, and conditional coverage error, measured as the mean absolute deviation of coverage from $1-\alpha$ over input cells (MACE). The target coverage is $90\%$. Implementation details and further comparisons are in Appendix~\ref{app:experiments}.

\begin{figure}[!t]
\centering
\includegraphics[width=\columnwidth]{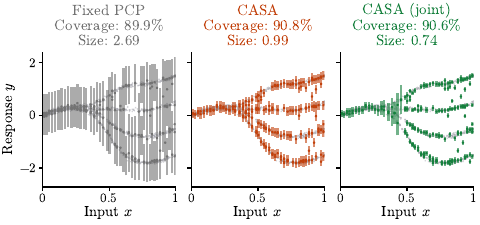}
\caption{Calibrated prediction sets on a synthetic DGP with $B=5$. Gray points are test pairs; at each displayed input, colored segments are the set and dots the generated samples.}
\label{fig:sets}
\end{figure}
\textbf{Synthetic benchmark.} In Figure~\ref{fig:frontiers}, we first compare \casa{} with fixed PCP on synthetic DGPs that isolate multimodality, heavy tails, and compact-to-extended responses. Against fixed PCP, both variants of \casa{} maintain marginal validity, produce smaller sets on most DGPs, and lower conditional coverage error on most DGPs without raising it on any. Panel~B shows that the gain holds in two directions on the emerging-modes DGP: at a fixed budget, \casa{} produces sets about $53\%$ smaller, and at the same set size, it uses about $25\%$ fewer samples. Figure~\ref{fig:sets} shows the calibrated sets on this DGP: fixed PCP's sets bridge the gaps between modes, whereas those of \casa{} stay at the scale of a mode, as Theorem~\ref{thm:threshold} predicts.

\begin{table*}[!t]
\centering\small
\caption{\casa{} combined with other generative conformal methods at $B=5$ and a $90\%$ coverage target. Arrows compare each method's fixed count with \casa{} count allocation. Size: mean set size; $\Delta$: paired geometric-mean size change; MACE in percentage points. Bold/$^\dagger$: significant improvement/deterioration (paired $95\%$ interval). $^\ddagger$C-PCP uses 20 extra auxiliary draws. Details are in Appendix~\ref{app:planned}.}
\label{tab:baselines}
\setlength{\tabcolsep}{3pt}
\begin{tabular}{lccc@{\hspace{8pt}}ccc@{\hspace{8pt}}ccc}\toprule
& \multicolumn{3}{c}{Rare modes} & \multicolumn{3}{c}{Emerging modes} & \multicolumn{3}{c}{Compact to extended} \\
\cmidrule(lr){2-4}\cmidrule(lr){5-7}\cmidrule(lr){8-10}
Method & Size & $\Delta$(\%) & MACE & Size & $\Delta$(\%) & MACE & Size & $\Delta$(\%) & MACE \\
\midrule
PCP & 1.31$\mkern-1mu\to\mkern-1mu$1.15 & \textbf{$-$12.0} & 14.6$\mkern-1mu\to\mkern-1mu$\textbf{12.9} & 3.29$\mkern-1mu\to\mkern-1mu$1.54 & \textbf{$-$54.9} & 10.0$\mkern-1mu\to\mkern-1mu$\textbf{7.4} & 2.73$\mkern-1mu\to\mkern-1mu$2.45 & \textbf{$-$10.2} & 9.6$\mkern-1mu\to\mkern-1mu$\textbf{9.3} \\
CP$^2$ & 2.97$\mkern-1mu\to\mkern-1mu$2.63 & \textbf{$-$11.5} & 1.6$\mkern-1mu\to\mkern-1mu$2.3$^\dagger$ & 4.90$\mkern-1mu\to\mkern-1mu$4.13 & \textbf{$-$16.2} & 2.9$\mkern-1mu\to\mkern-1mu$3.9$^\dagger$ & 3.68$\mkern-1mu\to\mkern-1mu$3.41 & \textbf{$-$7.6} & 1.5$\mkern-1mu\to\mkern-1mu$2.1$^\dagger$ \\
RCP & 2.00$\mkern-1mu\to\mkern-1mu$1.86 & \textbf{$-$7.2} & 3.2$\mkern-1mu\to\mkern-1mu$\textbf{2.3} & 2.80$\mkern-1mu\to\mkern-1mu$2.49 & \textbf{$-$11.9} & 3.9$\mkern-1mu\to\mkern-1mu$\textbf{2.5} & 2.59$\mkern-1mu\to\mkern-1mu$2.46 & \textbf{$-$5.1} & 3.2$\mkern-1mu\to\mkern-1mu$\textbf{2.4} \\
CP4Gen & 1.36$\mkern-1mu\to\mkern-1mu$1.18 & \textbf{$-$13.0} & 14.3$\mkern-1mu\to\mkern-1mu$\textbf{13.2} & 2.55$\mkern-1mu\to\mkern-1mu$2.60 & +1.5 & 4.8$\mkern-1mu\to\mkern-1mu$4.9 & 2.38$\mkern-1mu\to\mkern-1mu$2.43 & +2.3$^\dagger$ & 5.3$\mkern-1mu\to\mkern-1mu$5.5$^\dagger$ \\
C-PCP$^\ddagger$ & 2.08$\mkern-1mu\to\mkern-1mu$1.99 & \textbf{$-$4.3} & 2.6$\mkern-1mu\to\mkern-1mu$2.8 & 2.55$\mkern-1mu\to\mkern-1mu$1.96 & \textbf{$-$23.2} & 3.1$\mkern-1mu\to\mkern-1mu$2.9 & 2.62$\mkern-1mu\to\mkern-1mu$2.57 & \textbf{$-$2.0} & 3.0$\mkern-1mu\to\mkern-1mu$3.3 \\
\bottomrule\end{tabular}
\end{table*}

Table~\ref{tab:baselines} compares other generative conformal prediction methods with and without the count allocation of \casa{} on a subset of DGPs. \casa{} shrinks most of their sets and lowers conditional coverage error where the construction's radius does not already target it, consistent with Proposition~\ref{prop:oracle}. This reduction is not guaranteed: because \casa{} allocates counts to reduce size, it need not further reduce conditional coverage error when the construction's radius already targets it, as for CP$^2$.

\begin{figure}[!t]
\centering
\includegraphics[width=\linewidth]{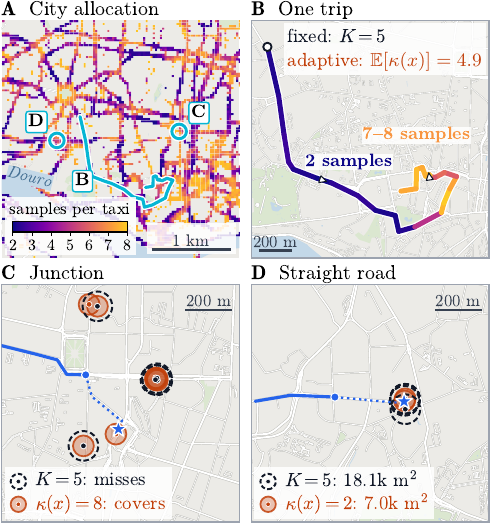}
\caption{Porto taxi routes at $B=5$. (A)~Mean allocated samples. (B)~Samples along one trip. \mbox{(C--D)}~More samples at a junction, fewer on a straight road. Stars mark true positions; dashed and filled circles are fixed PCP and \casa{} sets.}
\label{fig:porto}
\vspace{-0.2cm}
\end{figure}
\textbf{Real data: taxi routes.}\label{sec:porto} We present one real task here, taxi trajectory prediction on the Porto dataset \citep{porto2015trajectory}: from a taxi's recent GPS positions, we predict its position $300$\,m further along its route. The future is multimodal: a straight road leaves one likely position, whereas a junction leaves one per exit. The generator samples futures from similar training trips (Appendix~\ref{app:porto}). \casa{} allocates samples along the street geometry (Figure~\ref{fig:porto}): taxis on through-roads receive two samples and those near junctions up to eight, and the extra samples cover exits that fixed PCP misses. The shared calibrated radius then falls, so sets shrink even on straight roads. At the same budget, \casa{} keeps the target coverage with sets about $37\%$ smaller than fixed PCP, and the gain persists at equal realized samples, under a time split, and with a learned generator. We also consider two further real tasks with different response types (Appendices~\ref{app:protein} and~\ref{app:answers}). For protein backbone angles, \casa{} reduces set area by about $13\%$ at the same budget and $24\%$ at equal samples. For visual question answering with a vision-language model, where each response is a discrete answer, it matches fixed sampling's answer-set size with about $28\%$ fewer samples at an $80\%$ target.

\textbf{Ablation and sensitivity.} A natural question is whether the gain comes from matching counts to inputs or merely from letting the count vary over a wider range. To separate the two, we compare \casa{} with input-blind counts drawn uniformly from the same candidate set at the same budget with $B=5$ on the synthetic DGPs. Fixed PCP attains $90.0\%$ coverage with a MACE of $8.9$ percentage points. Random counts enlarge PCP uncertainty sets by $6\%$ and leave MACE nearly unchanged, whereas \casa{} shrinks the set size by $12\%$ and lowers MACE to $7.8$, with valid marginal coverage. Varying the count without regard to the input therefore does not reproduce the gain; it comes from giving samples where they are worth most.

Step~1 is also robust to the regression model that estimates $\widehat c_k$ and $\widehat v_k$: changing its number of neighbors or replacing it with a fixed partition of the input space keeps \casa{} valid and changes its mean set size by less than $1\%$, preserving the reduction over fixed PCP (Appendix~\ref{app:implementation}).

\vspace{-0.1cm}
\section{Conclusion}
\vspace{-.1cm}
We proposed \casa{}, which allocates generator samples across inputs to minimize set size under an average sampling budget while preserving finite-sample marginal coverage; theory and experiments show that reallocation pays off when a fixed count misses modes. Several limitations suggest extensions. First, the count is chosen from the input before any sample is drawn; letting it adapt to the samples already drawn could detect missed modes directly. Second, the theory assumes separated modes, and extending it to smooth or heavy-tailed distributions would clarify when allocation helps beyond representation. Third, allocation cannot repair a generator that never samples a relevant mode. Fourth, estimating coverage and size requires a one-time pilot pool, costly for expensive generators. Finally, \casa{} targets set size, so conditional coverage improves only as a by-product; targeting it directly, or extending to discrete outputs, is a natural next step.

\clearpage
\subsubsection*{Acknowledgments}
This work was supported by the National Science Foundation under Grant No.~CAIG-2425888.

\bibliographystyle{plainnat}
\bibliography{ref}

@inproceedings{wang2023pcp,
 title={Probabilistic Conformal Prediction Using Conditional Random Samples},
 author={Wang, Zhendong and Gao, Ruijiang and Yin, Mingzhang and Zhou, Mingyuan and Blei, David},
 booktitle={Proceedings of the 26th International Conference on Artificial Intelligence and Statistics},
 volume={206}, pages={8814--8836}, publisher={PMLR}, year={2023},
 url={https://proceedings.mlr.press/v206/wang23n.html}}

@article{small2014synoptic,
  title={A new synoptic scale resolving global climate simulation using the Community Earth System Model},
  author={Small, R Justin and Bacmeister, Julio and Bailey, David and Baker, Allison and Bishop, Stuart and Bryan, Frank and Caron, Julie and Dennis, John and Gent, Peter and Hsu, Hsiao-ming and others},
  journal={Journal of Advances in Modeling Earth Systems},
  volume={6},
  number={4},
  pages={1065--1094},
  year={2014},
  publisher={Wiley Online Library}
}

@inproceedings{
zhou2026calibrating,
title={Calibrating Decision Robustness via Inverse Conformal Risk Control},
author={Wenbin Zhou and Shixiang Zhu},
booktitle={Forty-third International Conference on Machine Learning},
year={2026},
url={https://openreview.net/forum?id=lV4tqcVIyx}
}

@article{kay2015cesm,
  title={The Community Earth System Model (CESM) large ensemble project: A community resource for studying climate change in the presence of internal climate variability},
  author={Kay, Jennifer E and Deser, Clara and Phillips, A and Mai, Angeline and Hannay, Cecile and Strand, Gary and Arblaster, Julie Michelle and Bates, Susan C and Danabasoglu, Gokhan and Edwards, James and others},
  journal={Bulletin of the American Meteorological Society},
  volume={96},
  number={8},
  pages={1333--1349},
  year={2015}
}

@inproceedings{plassier2025cp2,
 title={Probabilistic Conformal Prediction with Approximate Conditional Validity},
 author={Plassier, Vincent and Fishkov, Alexander and Guizani, Mohsen and Panov, Maxim and Moulines, Eric},
 booktitle={International Conference on Learning Representations}, year={2025}}

@inproceedings{plassier2025rcp,
 title={Rectifying Conformity Scores for Better Conditional Coverage},
 author={Plassier, Vincent and Fishkov, Alexander and Dheur, Victor and Guizani, Mohsen and {Ben Taieb}, Souhaib and Panov, Maxim and Moulines, Eric},
 booktitle={Proceedings of the 42nd International Conference on Machine Learning}, year={2025},
 url={https://arxiv.org/abs/2502.16336}}

@inproceedings{dheur2025unified,
 title={A Unified Comparative Study with Generalized Conformity Scores for Multi-Output Conformal Regression},
 author={Dheur, Victor and Fontana, Matteo and Estievenart, Yorick and Desobry, Naomi and {Ben Taieb}, Souhaib},
 booktitle={Proceedings of the 42nd International Conference on Machine Learning},
 series={Proceedings of Machine Learning Research}, volume={267}, pages={13444--13485},
 publisher={PMLR}, year={2025}, url={https://proceedings.mlr.press/v267/dheur25a.html}}

@article{zheng2024generative,
 title={Generative Conformal Prediction with Vectorized Non-Conformity Scores},
 author={Zheng, Minxing and Zhu, Shixiang},
 journal={arXiv preprint arXiv:2410.13735}, year={2024},
 doi={10.48550/arXiv.2410.13735},
 url={https://arxiv.org/abs/2410.13735}}

@article{yang2026cp4gen,
 title={Conformal Prediction for Generative Models via Adaptive Cluster-Based Density Estimation},
 author={Yang, Qidong and Zhu, Qianyu Julie and Giezendanner, Jonathan and Marzouk, Youssef and Bates, Stephen and Wang, Sherrie},
 journal={arXiv preprint arXiv:2601.22298}, year={2026},
 url={https://arxiv.org/abs/2601.22298}}

@article{barber2021limits,
 title={The Limits of Distribution-Free Conditional Predictive Inference},
 author={Barber, Rina Foygel and Cand\`es, Emmanuel J. and Ramdas, Aaditya and Tibshirani, Ryan J.},
 journal={Information and Inference}, volume={10}, number={2}, pages={455--482}, year={2021}}

@article{chernozhukov2021distributional,
 title={Distributional Conformal Prediction},
 author={Chernozhukov, Victor and W{\"u}thrich, Kaspar and Zhu, Yinchu},
 journal={Proceedings of the National Academy of Sciences},
 volume={118}, number={48}, pages={e2107794118}, year={2021},
 url={https://arxiv.org/abs/1909.07889}}

@book{vovk2005algorithmic,
 title={Algorithmic Learning in a Random World},
 author={Vovk, Vladimir and Gammerman, Alexander and Shafer, Glenn},
 publisher={Springer}, address={New York}, year={2005},
 doi={10.1007/b106715}}

@article{lei2018distribution,
 title={Distribution-Free Predictive Inference for Regression},
 author={Lei, Jing and G'Sell, Max and Rinaldo, Alessandro and Tibshirani, Ryan J. and Wasserman, Larry},
 journal={Journal of the American Statistical Association}, volume={113}, number={523},
 pages={1094--1111}, year={2018}, doi={10.1080/01621459.2017.1307116}}

@article{angelopoulos2023conformal,
 title={Conformal Prediction: A Gentle Introduction},
 author={Angelopoulos, Anastasios N. and Bates, Stephen},
 journal={Foundations and Trends in Machine Learning}, volume={16}, number={4},
 pages={494--591}, year={2023}, doi={10.1561/2200000101}}

@article{lei2015functional,
 title={A Conformal Prediction Approach to Explore Functional Data},
 author={Lei, Jing and Rinaldo, Alessandro and Wasserman, Larry},
 journal={Annals of Mathematics and Artificial Intelligence}, volume={74}, number={1--2},
 pages={29--43}, year={2015}, doi={10.1007/s10472-013-9366-6}}

@article{diquigiovanni2022functional,
 title={Conformal Prediction Bands for Multivariate Functional Data},
 author={Diquigiovanni, Jacopo and Fontana, Matteo and Vantini, Simone},
 journal={Journal of Multivariate Analysis}, volume={189}, pages={104879}, year={2022},
 doi={10.1016/j.jmva.2021.104879}}

@inproceedings{zaffran2022adaptive,
 title={Adaptive Conformal Predictions for Time Series},
 author={Zaffran, Margaux and F{\'e}ron, Olivier and Goude, Yannig and Josse, Julie and Dieuleveut, Aymeric},
 booktitle={Proceedings of the 39th International Conference on Machine Learning},
 series={Proceedings of Machine Learning Research}, volume={162}, pages={25834--25866},
 publisher={PMLR}, year={2022}, url={https://proceedings.mlr.press/v162/zaffran22a.html}}

@inproceedings{xu2024multidimensional,
 title={Conformal Prediction for Multi-Dimensional Time Series by Ellipsoidal Sets},
 author={Xu, Chen and Jiang, Hanyang and Xie, Yao},
 booktitle={Proceedings of the 41st International Conference on Machine Learning},
 series={Proceedings of Machine Learning Research}, volume={235}, pages={55076--55099},
 publisher={PMLR}, year={2024}, url={https://proceedings.mlr.press/v235/xu24m.html}}

@inproceedings{javanmardi2023partial,
 title={Conformal Prediction with Partially Labeled Data},
 author={Javanmardi, Alireza and Sale, Yusuf and Hofman, Paul and H{\"u}llermeier, Eyke},
 booktitle={Proceedings of the Twelfth Symposium on Conformal and Probabilistic Prediction with Applications},
 series={Proceedings of Machine Learning Research}, volume={204}, pages={251--266},
 publisher={PMLR}, year={2023}, url={https://proceedings.mlr.press/v204/javanmardi23a.html}}

@inproceedings{alaa2023meta,
 title={Conformal Meta-Learners for Predictive Inference of Individual Treatment Effects},
 author={Alaa, Ahmed M. and Ahmad, Zaid and van der Laan, Mark},
 booktitle={Advances in Neural Information Processing Systems}, volume={36},
 pages={47682--47703}, year={2023}}

@article{harris2026flow,
 title={Flow-Based Conformal Predictive Distributions},
 author={Harris, Trevor}, journal={arXiv preprint arXiv:2602.07633}, year={2026},
 url={https://arxiv.org/abs/2602.07633}}

@inproceedings{shahrokhi2025gps,
 title={Conformal Prediction Sets for Deep Generative Models via Reduction to Conformal Regression},
 author={Shahrokhi, Hooman and Roy, Devjeet Raj and Yan, Yan and Arnaoudova, Venera and Doppa, Jana},
 booktitle={Proceedings of the Forty-First Conference on Uncertainty in Artificial Intelligence},
 series={Proceedings of Machine Learning Research}, volume={286}, pages={3718--3748},
 publisher={PMLR}, year={2025}, url={https://proceedings.mlr.press/v286/shahrokhi25a.html}}

@article{graves2016adaptive,
 title={Adaptive Computation Time for Recurrent Neural Networks},
 author={Graves, Alex}, journal={arXiv preprint arXiv:1603.08983}, year={2016},
 url={https://arxiv.org/abs/1603.08983}}

@inproceedings{zuo2026strategic,
 title={Strategic Scaling of Test-Time Compute: A Bandit Learning Approach},
 author={Zuo, Bowen and Zhu, Yinglun},
 booktitle={International Conference on Learning Representations}, year={2026}}

@article{lin2023esm,
 title={Evolutionary-Scale Prediction of Atomic-Level Protein Structure with a Language Model},
 author={Lin, Zeming and Akin, Halil and Rao, Roshan and Hie, Brian and Zhu, Zhongkai and Lu, Wenting and Smetanin, Nikita and Verkuil, Robert and Kabeli, Ori and Shmueli, Yaniv and others},
 journal={Science}, volume={379}, number={6637}, pages={1123--1130}, year={2023}}

@inproceedings{ingraham2019generative,
 title={Generative Models for Graph-Based Protein Design},
 author={Ingraham, John and Garg, Vikas K. and Barzilay, Regina and Jaakkola, Tommi},
 booktitle={Advances in Neural Information Processing Systems}, year={2019}}

@inproceedings{singh2019textvqa,
 title={Towards {VQA} Models That Can Read},
 author={Singh, Amanpreet and Natarajan, Vivek and Shah, Meet and Jiang, Yu and Chen, Xinlei and Batra, Dhruv and Parikh, Devi and Rohrbach, Marcus},
 booktitle={Proceedings of the IEEE/CVF Conference on Computer Vision and Pattern Recognition}, year={2019}}

@inproceedings{marino2019okvqa,
 title={{OK-VQA}: A Visual Question Answering Benchmark Requiring External Knowledge},
 author={Marino, Kenneth and Rastegari, Mohammad and Farhadi, Ali and Mottaghi, Roozbeh},
 booktitle={Proceedings of the IEEE/CVF Conference on Computer Vision and Pattern Recognition}, year={2019}}

@inproceedings{su2024api,
 title={{API} Is Enough: Conformal Prediction for Large Language Models Without Logit-Access},
 author={Su, Jiayuan and Luo, Jing and Wang, Hongwei and Cheng, Lu},
 booktitle={Findings of the Association for Computational Linguistics: EMNLP}, year={2024}}

@article{bai2025qwen25vl,
 title={{Qwen2.5-VL} Technical Report},
 author={Bai, Shuai and Chen, Keqin and Liu, Xuejing and Wang, Jialin and Ge, Wenbin and Song, Sibo and Dang, Kai and Wang, Peng and Wang, Shijie and Tang, Jun and others},
 journal={arXiv preprint arXiv:2502.13923}, year={2025}}

@article{karmarkar1984new,
 title={A New Polynomial-Time Algorithm for Linear Programming},
 author={Karmarkar, Narendra},
 journal={Combinatorica}, volume={4}, number={4}, pages={373--395}, year={1984}}

@inproceedings{ho2020denoising,
 title={Denoising Diffusion Probabilistic Models},
 author={Ho, Jonathan and Jain, Ajay and Abbeel, Pieter},
 booktitle={Advances in Neural Information Processing Systems}, volume={33}, pages={6840--6851}, year={2020}}

@inproceedings{brown2020language,
 title={Language Models are Few-Shot Learners},
 author={Brown, Tom B. and Mann, Benjamin and Ryder, Nick and Subbiah, Melanie and Kaplan, Jared and Dhariwal, Prafulla and others},
 booktitle={Advances in Neural Information Processing Systems}, volume={33}, pages={1877--1901}, year={2020}}

@article{shafer2008tutorial,
  title={A tutorial on conformal prediction},
  author={Shafer, Glenn and Vovk, Vladimir},
  journal={Journal of Machine Learning Research},
  volume={9},
  pages={371--421},
  year={2008}
}

@inproceedings{romano2019conformalized,
  title={Conformalized quantile regression},
  author={Romano, Yaniv and Patterson, Evan and Cand{\`e}s, Emmanuel},
  booktitle={Advances in Neural Information Processing Systems},
  volume={32},
  year={2019}
}

@inproceedings{tibshirani2019conformal,
  title={Conformal prediction under covariate shift},
  author={Tibshirani, Ryan J and Barber, Rina Foygel and Cand{\`e}s, Emmanuel J and Ramdas, Aaditya},
  booktitle={Advances in Neural Information Processing Systems},
  volume={32},
  pages={2526--2536},
  year={2019}
}

@article{barber2023conformal,
  title={Conformal prediction beyond exchangeability},
  author={Barber, Rina Foygel and Cand{\`e}s, Emmanuel J and Ramdas, Aaditya and Tibshirani, Ryan J},
  journal={The Annals of Statistics},
  volume={51},
  number={2},
  pages={816--845},
  year={2023}
}

@article{bates2021distribution,
  title={Distribution-free, risk-controlling prediction sets},
  author={Bates, Stephen and Angelopoulos, Anastasios N and Lei, Lihua and Malik, Jitendra and Jordan, Michael I},
  journal={Journal of the ACM},
  volume={68},
  number={6},
  pages={1--34},
  year={2021}
}

@inproceedings{angelopoulos2024conformal,
  title={Conformal risk control},
  author={Angelopoulos, Anastasios N and Bates, Stephen and Fisch, Adam and Lei, Lihua and Schuster, Tal},
  booktitle={International Conference on Learning Representations},
  year={2024}
}

@inproceedings{gibbs2021adaptive,
  title={Adaptive conformal inference under distribution shift},
  author={Gibbs, Isaac and Cand{\`e}s, Emmanuel},
  booktitle={Advances in Neural Information Processing Systems},
  volume={34},
  pages={1660--1672},
  year={2021}
}

@inproceedings{xu2021conformal,
  title={Conformal prediction interval for dynamic time-series},
  author={Xu, Chen and Xie, Yao},
  booktitle={International Conference on Machine Learning},
  pages={11559--11569},
  year={2021}
}

@article{harris2025locally,
  title={Locally adaptive conformal inference for operator models},
  author={Harris, Trevor and Liu, Yan},
  journal={arXiv preprint arXiv:2507.20975},
  year={2025}
}

@article{einbinder2024label,
  author  = {Bat-Sheva Einbinder and Shai Feldman and Stephen Bates and Anastasios N. Angelopoulos and Asaf Gendler and Yaniv Romano},
  title   = {Label Noise Robustness of Conformal Prediction},
  journal = {Journal of Machine Learning Research},
  year    = {2024},
  volume  = {25},
  number  = {328},
  pages   = {1--66},
  url     = {https://jmlr.org/papers/v25/23-1549.html}
}

@article{zheng2026latent,
  title={Beyond Prediction: Conformal Inference for Latent Distributional Parameters},
  author={Zheng, Minxing and Zhou, Wenbin and Zhu, Shixiang},
  journal={arXiv preprint arXiv:2608.03607},
  year={2026},
  url={https://arxiv.org/abs/2608.03607}
}

@article{izbicki2022cd,
  title={{CD}-split and {HPD}-split: Efficient conformal regions in high dimensions},
  author={Izbicki, Rafael and Shimizu, Gilson and Stern, Rafael B},
  journal={Journal of Machine Learning Research},
  volume={23},
  number={87},
  pages={1--32},
  year={2022}
}

@inproceedings{noorani2025cpq,
 title={Conformal Prediction Beyond the Seen: A Missing Mass Perspective for Uncertainty Quantification in Generative Models},
 author={Noorani, Sima and Kiyani, Shayan and Pappas, George and Hassani, Hamed},
 booktitle={Advances in Neural Information Processing Systems}, year={2025}}

@inproceedings{damani2025hard,
 title={Learning How Hard to Think: Input-Adaptive Allocation of {LM} Computation},
 author={Damani, Mehul and Shenfeld, Idan and Peng, Andi and Bobu, Andreea and Andreas, Jacob},
 booktitle={International Conference on Learning Representations}, year={2025}}

@inproceedings{quach2024clm,
 title={Conformal Language Modeling},
 author={Quach, Victor and Fisch, Adam and Schuster, Tal and Yala, Adam and Sohn, Jae Ho and Jaakkola, Tommi and Barzilay, Regina},
 booktitle={International Conference on Learning Representations}, year={2024}}

@inproceedings{snell2025scaling,
 title={Scaling {LLM} Test-Time Compute Optimally Can be More Effective than Scaling Parameters for Reasoning},
 author={Snell, Charlie and Lee, Jaehoon and Xu, Kelvin and Kumar, Aviral},
 booktitle={International Conference on Learning Representations}, year={2025},
 url={https://arxiv.org/abs/2408.03314}}

@inproceedings{wang2023self,
 title={Self-Consistency Improves Chain of Thought Reasoning in Language Models},
 author={Wang, Xuezhi and Wei, Jason and Schuurmans, Dale and Le, Quoc and Chi, Ed and Narang, Sharan and Chowdhery, Aakanksha and Zhou, Denny},
 booktitle={International Conference on Learning Representations}, year={2023},
 url={https://arxiv.org/abs/2203.11171}}

@article{gibbs2025conformal,
 title={Conformal Prediction with Conditional Guarantees},
 author={Gibbs, Isaac and Cherian, John J. and Cand{\`e}s, Emmanuel J.},
 journal={Journal of the Royal Statistical Society Series B: Statistical Methodology}, year={2025}}

@article{neyman1934two,
 title={On the Two Different Aspects of the Representative Method: The Method of Stratified Sampling and the Method of Purposive Selection},
 author={Neyman, Jerzy},
 journal={Journal of the Royal Statistical Society}, volume={97}, number={4}, pages={558--625}, year={1934}}

@article{bugallo2017adaptive,
 title={Adaptive Importance Sampling: The Past, the Present, and the Future},
 author={Bugallo, M{\'o}nica F. and Elvira, V{\'\i}ctor and Martino, Luca and Luengo, David and M{\'\i}guez, Joaqu{\'\i}n and Djuri{\'c}, Petar M.},
 journal={IEEE Signal Processing Magazine}, volume={34}, number={4}, pages={60--79}, year={2017}}

@inproceedings{audibert2010best,
 title={Best Arm Identification in Multi-Armed Bandits},
 author={Audibert, Jean-Yves and Bubeck, S{\'e}bastien and Munos, R{\'e}mi},
 booktitle={Conference on Learning Theory}, pages={41--53}, year={2010}}

@misc{porto2015trajectory,
 title={Taxi Service Trajectory -- Prediction Challenge, ECML PKDD 2015},
 author={{UCI Machine Learning Repository}},
 howpublished={UCI Machine Learning Repository},
 year={2015},
 doi={10.24432/C55W25},
 url={https://doi.org/10.24432/C55W25}}

\clearpage
\appendix

\onecolumn

\section{Proofs and Theoretical Results}
\label{app:proofs}

\subsection{Estimation and Coverage}
Throughout this appendix, the radius rule $r(\cdot)$ is fixed, $\mathcal B(y)$ is the ball of that radius around $y$, and $m(x;y)$, $h_k(x;y)$, $c_k(x)$, and $v_k(x)$ are as in Section~\ref{sec:formulation}. When the radius $r$ also varies, as in the joint variant, we add it as an argument, writing $m(x;y;r)$, $h_k(x;y;r)$, $c_k(x;r)$, and $v_k(x;r)$.

\paragraph{Diminishing sampling increments.}
Define $\Delta c_k(x)=c_{k+1}(x)-c_k(x)$ and $\Delta v_k(x)=v_{k+1}(x)-v_k(x)$. Then
\begin{equation}
 \begin{aligned}
 \Delta c_k(x)&=\int m(x;y)^k\{1-m(x;y)\}\,\d\mathcal P_x(y),\\
 \Delta v_k(x)&=\int m(x;y)^k\{1-m(x;y)\}\,\d y.
 \end{aligned}
 \label{eq:increment}
\end{equation}
To prove this, and that both increments are nonnegative and nonincreasing in $k$, fix $x$ and $y$ and write $m=m(x;y)\in[0,1]$. By \eqref{eq:inclusion}, $\Delta h_k(x;y)=h_{k+1}(x;y)-h_k(x;y)=m^k(1-m)$, the probability that the first $k$ draws miss $\mathcal B(y)$ and the next one falls in it. This is nonnegative and nonincreasing in $k$. Integrating against $\mathcal P_x$ and against volume proves the identities and the two shape properties; volume calculations use Tonelli's theorem, and finite expected volume ensures that the differences are defined. \hfill$\square$

\paragraph{Rao--Blackwell estimation.}
\begin{proposition}[Unbiasedness and variance reduction]
\label{prop:rb}
For fixed $(X_i,Y_i)$ and $k\le N$, the estimator $\widehat h_k(X_i;y)$ in \eqref{eq:rb} is unbiased over the pool for $h_k(X_i;y)$. Consequently, $\widehat c_k(X_i)$ in \eqref{eq:perinput} is unbiased for $h_k(X_i;Y_i)$, whose average over $Y_i\sim\mathcal P_{X_i}$ is $c_k(X_i)$, and $\widehat v_k(X_i)$ is unbiased for $v_k(X_i)$ when it is finite. Both are conditional expectations, given the pool, of the coverage indicator and the volume of a uniformly selected $k$-subset, and therefore have no larger variance than those statistics. Their curves in $k$, and their regressions in Step~1 with nonnegative weights, are nondecreasing and discretely concave.
\end{proposition}
\emph{Proof.} Condition on $(X_i,Y_i)$ and on the unordered pool, and choose a uniformly random $k$-subset $A$ of it. Exactly $\binom{N-H(X_i;y)}{k}$ of the $\binom Nk$ subsets contain no sample within radius $r$ of $y$, so $\widehat h_k(X_i;y)$ is the conditional probability that $A$ covers $y$. Taking $y=Y_i$ gives the statement for $\widehat c_k(X_i)$, and integrating over $y$ gives it for $\widehat v_k(X_i)$. Before conditioning on the pool, the selected samples are $k$ independent draws from $\widehat f(X_i)$, which gives unbiasedness, and the law of total variance gives the variance comparison whenever second moments exist. For the shape property, fix a value $H\ge1$ of the count and let $a(k)=1-\binom{N-H}{k}/\binom Nk$; if $H=0$, then $a(k)=0$ for all $k$. For $0\le k<N$,
\[
 a(k+1)-a(k)=\frac{\binom{N-H}{k}}{\binom Nk}\frac{H}{N-k}\ge0,
\]
and successive positive increments have ratio $(N-H-k)/(N-k-1)\le1$, with all increments zero once the miss probability vanishes. So $a(k)$ is nondecreasing and discretely concave, and integration and nonnegative regression weights preserve both properties. \hfill$\square$

\paragraph{Finite-sample coverage.}
\emph{Proof of Theorem~\ref{thm:validity}.} Condition on the generator's training data, on $\mathcal D_m$, and on all quantities fitted from them, including $\widehat\kappa$. Augment each calibration pair and the test pair with an independent random seed that determines its count $\widehat\kappa(X)$, if the rule is randomized, and its generated samples. These augmented observations are exchangeable. The frozen procedure that draws $\widehat\kappa(X)$ samples $\widehat Y_1,\ldots,\widehat Y_{\widehat\kappa(X)}$ and computes $S=\min_{j\le\widehat\kappa(X)}\|Y-\widehat Y_j\|_2/r(X)$ is the same measurable map for every observation, so the scores are exchangeable, and the conformal order-statistic argument gives $\bP\{S_{\mathrm{test}}\le\widehat q\}\ge1-\alpha$. This event is exactly $Y_{\mathrm{test}}\in\widehat\cC_\alpha(X_{\mathrm{test}})$, the set \eqref{eq:pcpset} with $K=\widehat\kappa(X_{\mathrm{test}})$. Averaging over the conditioning information proves the theorem. The same argument applies to the joint variant with the learned radius $\widehat r(x)$ in place of $r(x)$. \hfill$\square$

\subsection{Optimization}
\label{app:dual}
Throughout this subsection, write $t=1-\alpha$. We state the results for the joint variant, in which each input chooses a count $k\in\cK$ and a radius $r$ from a finite grid $\cR$; count-only \casa{} is the special case of a single radius, $|\cR|=1$. Abbreviate a pair as $a=(k,r)$, with $k(a)$ its count and $\cA=\cK\times\cR$ the set of pairs, and write $\widehat c_i(a)=\widehat c_k(X_i;r)$, $\widehat v_i(a)=\widehat v_k(X_i;r)$, and $g_i(a)=\widehat v_i(a)-\mu\,\widehat c_i(a)+\lambda k(a)$ for the tables and priced costs, so that $g_i(a)=g_k(X_i)$ in the count-only case. The allocation set has $m$ observations.

\paragraph{The empirical allocation problem.}
Let $\pi_i(k,r)=\pi(k,r\mid X_i)$ be a distribution over the finite candidate pairs at input $X_i$. The estimated counterpart of \eqref{eq:population} is
\begin{equation}
 \begin{aligned}
 \min_{\pi_1,\ldots,\pi_m}\quad&\frac1m\sum_{i,a}\pi_i(a)\,\widehat v_i(a)\\
 \text{s.t.}\quad&\frac1m\sum_{i,a}\pi_i(a)\,\widehat c_i(a)\ge1-\alpha,\\
 &\frac1m\sum_{i,a}\pi_i(a)\,k(a)\le B,
 \end{aligned}
 \label{eq:allocation}
\end{equation}
where the sums run over $k\in\cK$ and $r\in\cR$, with $\pi_i(k,r)\ge0$ and $\sum_{k,r}\pi_i(k,r)=1$ for each $i$. Write $\widehat V^*$ for its optimal value and $\widehat V(\pi)$, $\widehat C(\pi)$, and $\widehat K(\pi)$ for its mean estimated volume, coverage, and count. Only the coverage and budget constraints couple the inputs. The dual prices these constraints and minimizes over each input's distribution separately. When minimizers tie, mixing them at no more than two inputs gives a feasible optimal policy (Proposition~\ref{prop:dual}).

\paragraph{Gain over fixed PCP.}
The dual multipliers also characterize the improvement over a fixed allocation in the empirical program.
\begin{proposition}[Gain over fixed PCP]
\label{prop:gain}
Suppose \eqref{eq:allocation} is feasible, with optimal value $\widehat V^*$, and write $\pi_i(k,r)=\pi(k,r\mid X_i)$. There exist optimal dual multipliers $\mu,\lambda\ge0$ such that every policy satisfies
\begin{equation}
 \begin{aligned}
 \widehat V(\pi)-\widehat V^*
 &=\frac1m\sum_{i,a}\pi_i(a)
   \Bigl\{g_i(a)-\min_{a'}g_i(a')\Bigr\}\\
 &\quad+\mu\bigl\{\widehat C(\pi)-(1-\alpha)\bigr\}\\
 &\quad+\lambda\bigl\{B-\widehat K(\pi)\bigr\},
 \end{aligned}
 \label{eq:certificate}
\end{equation}
where $\widehat V(\pi)$, $\widehat C(\pi)$, and $\widehat K(\pi)$ are the mean estimated volume, coverage, and count of $\pi$ over $\mathcal D_m$; for a feasible policy all three terms are nonnegative. In particular, fixed PCP, which uses count $B$ and one radius $r$ at every input, with $r$ meeting the coverage target exactly, exceeds $\widehat V^*$ by its priced regret $\widehat R_F=m^{-1}\sum_i\{g_i(B,r)-\min_{k',r'}g_i(k',r')\}$, which is zero if and only if $(B,r)$ is the cheapest choice at every input.
\end{proposition}
For a fixed comparator meeting both constraints with equality, the gain equals its average excess local Lagrangian cost. The same identity holds with population tables when optimal multipliers exist. The regret also converts into computation: the allocation matches fixed PCP's volume with $\Delta$ fewer draws per prediction whenever $\Delta\,\lambda(B-\Delta)\le\widehat R_F$, with $\lambda(B')$ the price at budget $B'$. The identity concerns the estimated program before calibration.

\paragraph{The two prices.}
The prices of Proposition~\ref{prop:gain} solve the dual of \eqref{eq:allocation}, which makes the program cheap to solve and gives the prices a meaning.
\begin{proposition}[Exact two-price dual]
\label{prop:dual}
Let $D(\mu,\lambda)=\mu(1-\alpha)-\lambda B+m^{-1}\sum_{i=1}^m\min_{k,r}g_i(k,r)$, and suppose \eqref{eq:allocation} is feasible. (i)~$D$ is concave and piecewise linear, $\widehat V^*=\max_{\mu,\lambda\ge0}D(\mu,\lambda)$ with the maximum attained, and a feasible policy is optimal if and only if each $\pi_i$ is supported on minimizers of $g_i$, the coverage constraint binds when $\mu>0$, and the budget binds when $\lambda>0$; some optimal policy is deterministic at all but at most two learning inputs. (ii)~For fixed $\mu$, minimizing $\widehat V(\pi)-\mu\widehat C(\pi)$ under the budget alone is solved exactly by water-filling: take at each input the best radius for every count, start every input at the smallest count, and add draws across inputs in decreasing order of the reduction in priced cost per draw, along the lower convex hull of each input's cost in the count, until the budget is spent or no further reduction is possible; $\lambda(\mu)$ is the reduction at the last draw added when the budget binds and zero otherwise. The coverage of this solution is nondecreasing in $\mu$, so bisection on $\mu$ over a bracket $[0,M]$ returns a feasible policy with volume at most $\widehat V^*+\alpha\epsilon$ after $O(\log(M/\epsilon))$ passes of $O(m|\cK||\cR|+m|\cK|\log(m|\cK|))$ operations each.
\end{proposition}
For \casa{}, which allocates counts only, part~(i) is Theorem~\ref{thm:dual}: introducing $u_i=\min_{k\in\cK}g_k(X_i)$ turns $\max_{\mu,\lambda\ge0}D(\mu,\lambda)$ into the linear program \eqref{eq:dualLP}, and its optimality condition says that an optimal allocation uses at each input only counts minimizing $g_k(X_i)$, which defines the rule $\widehat\kappa$.
Any exact method must read the $m|\cK||\cR|$ table entries, so the solver is optimal up to the logarithmic factor. Randomization is not what makes the allocation work: the optimal rule is deterministic except where mixing two choices meets a constraint exactly. The computation price is the marginal value of a draw: $\widehat V^*(B)$ is convex and nonincreasing in the budget with $-\lambda$ as a subgradient, so one more call per prediction saves at most $\lambda$ in estimated volume, and $\lambda=0$ when the budget goes unspent, as under \casa{} when the shared radius already covers the generator's mass at every input.

\paragraph{Proof of Propositions~\ref{prop:gain} and~\ref{prop:dual}(i).}
The Lagrangian of \eqref{eq:allocation} with multipliers $\mu,\lambda\ge0$ is
\[
 \mu t-\lambda B+\frac1m\sum_{i=1}^m\sum_{a\in\cA}\pi_i(a)\,g_i(a).
\]
Minimizing over each probability simplex places mass on minimizers of $g_i$ and gives the dual objective $D(\mu,\lambda)$. \emph{Concavity and duality.} Each $\min_a g_i(a)$ is a minimum of finitely many affine functions of $(\mu,\lambda)$, so $D$ is concave and piecewise linear. The feasible set of \eqref{eq:allocation} is a nonempty polytope, so linear-programming strong duality gives equal optimal values and a dual maximizer. \emph{Optimality.} For any $\mu,\lambda\ge0$ and any policy, substituting $\widehat v_i(a)=g_i(a)+\mu\widehat c_i(a)-\lambda k(a)$ and $\sum_a\pi_i(a)=1$ gives
\[
 \begin{aligned}
 \widehat V(\pi)-D(\mu,\lambda)
 &=\frac1m\sum_{i,a}\pi_i(a)
   \Bigl\{g_i(a)-\min_{a'}g_i(a')\Bigr\}\\
 &\quad+\mu\bigl\{\widehat C(\pi)-t\bigr\}\\
 &\quad+\lambda\bigl\{B-\widehat K(\pi)\bigr\}.
 \end{aligned}
\]
At a dual maximizer $D(\mu,\lambda)=\widehat V^*$ by~(i), which is \eqref{eq:certificate}. For a feasible policy the three terms are nonnegative, so it is optimal exactly when all three vanish, which is the stated condition, and the identity at the maximizer is Proposition~\ref{prop:gain}. For the deterministic-policy claim, adding a slack variable to each of the two inequality constraints puts \eqref{eq:allocation} in standard form with $m+2$ equality constraints. A bounded feasible linear program attains its optimum at a basic feasible solution, which has at most $m+2$ positive variables. Each simplex needs at least one positive $\pi_i(a)$, which leaves at most two further positive entries: either two inputs mix two actions or one input mixes three. \emph{Monotonicity in the prices.} Let $a$ minimize $g_i$ at $(\mu,\lambda)$ and $a'$ at $(\mu',\lambda)$ with $\mu'>\mu$. Adding the two optimality inequalities cancels the volume and count terms and leaves $(\mu'-\mu)\{\widehat c_i(a')-\widehat c_i(a)\}\ge0$. The same argument in $\lambda$ at fixed $\mu$ gives $(\lambda'-\lambda)\{k(a')-k(a)\}\le0$. \hfill$\square$

\paragraph{Proof of Proposition~\ref{prop:dual}(ii).}
Fix $\mu\ge0$, write $w_i(a)=\widehat v_i(a)-\mu\,\widehat c_i(a)$, and let $q(\mu)=\min\{\widehat V(\pi)-\mu(\widehat C(\pi)-t):\widehat K(\pi)\le B\}$, the Lagrangian of \eqref{eq:allocation} with only the coverage constraint priced. Its minimizers are those of the problem in the proposition. \emph{Radius.} Replacing each action $(k,r)$ at input $i$ by $(k,r_i(k))$, where $r_i(k)$ minimizes $w_i(k,\cdot)$, keeps every count and does not increase the objective, so it suffices to work with $M_i(k)=\min_r w_i(k,r)$, computed in one pass over the tables. \emph{Counts.} A distribution over counts at input $i$ with mean $\kappa$ costs at least $\breve M_i(\kappa)$, the lower convex envelope of $k\mapsto M_i(k)$ over the menu, and attains it by mixing the two hull vertices adjacent to $\kappa$. The problem is therefore $\min m^{-1}\sum_i\breve M_i(\kappa_i)$ subject to $m^{-1}\sum_i\kappa_i\le B$ and $\kappa_i\in[k_{\min},k_{\max}]$, a separable convex allocation of one resource. Starting every $\kappa_i$ at $k_{\min}$ and raising the $\kappa_i$ along the segments of the $\breve M_i$ in increasing order of slope, until the budget is spent or no segment has negative slope, satisfies its optimality conditions with multiplier $\lambda(\mu)$ equal to minus the slope of the last segment used: every segment used has slope at most $-\lambda(\mu)$, every segment not used has slope at least $-\lambda(\mu)$, and convexity makes each input's segments come in order. At most one segment is used partially, so at most one input mixes two counts. The hulls take $O(m|\cK|)$ operations, since the counts are sorted, and ordering the at most $m(|\cK|-1)$ segments takes $O(m|\cK|\log(m|\cK|))$.
\emph{Bisection.} Let $\pi_\mu$ be the water-filling solution. For every $\mu'$, $q(\mu')\le\widehat V(\pi_\mu)-\mu'\{\widehat C(\pi_\mu)-t\}=q(\mu)+(\mu'-\mu)\{t-\widehat C(\pi_\mu)\}$, so $t-\widehat C(\pi_\mu)$ is a supergradient of the concave function $q$. Supergradients of a concave function are nonincreasing, so $\widehat C(\pi_\mu)$ is nondecreasing in $\mu$, and linear-programming duality gives $\max_{\mu\ge0}q(\mu)=\widehat V^*$. If $\widehat C(\pi_0)\ge t$, then $\pi_0$ is feasible with $\widehat V(\pi_0)=q(0)\le\widehat V^*$ and is optimal. Otherwise, because \eqref{eq:allocation} is feasible, beyond the last breakpoint of $\mu\mapsto\pi_\mu$ the solution maximizes coverage under the budget, so $\widehat C(\pi_M)\ge t$ for a finite $M$, which doubling finds. Bisection keeps $\mu_{\rm lo}<\mu_{\rm hi}$ with $\widehat C(\pi_{\mu_{\rm lo}})<t\le\widehat C(\pi_{\mu_{\rm hi}})$ until $\mu_{\rm hi}-\mu_{\rm lo}\le\epsilon$. Choose $\theta\in[0,1]$ with $\theta\widehat C(\pi_{\mu_{\rm lo}})+(1-\theta)\widehat C(\pi_{\mu_{\rm hi}})=t$; the mixture $\bar\pi$ meets both constraints. With $\Delta=(1-\theta)\{\widehat C(\pi_{\mu_{\rm hi}})-t\}\in[0,\alpha]$,
\[
 \widehat V(\bar\pi)=\theta q(\mu_{\rm lo})+(1-\theta)q(\mu_{\rm hi})+(\mu_{\rm hi}-\mu_{\rm lo})\Delta\le\widehat V^*+\alpha\epsilon.
\]
Once the bracket contains no breakpoint of $\mu\mapsto\pi_\mu$ other than a maximizer $\mu^*$ of $q$, both bracket solutions minimize the priced objective at $\mu^*$, so $\bar\pi$ does too, meets the coverage target with equality, and is exactly optimal. The number of steps is $O(\log(M/\epsilon))$, and each costs $O(m|\cA|+m|\cK|\log(m|\cK|))$. Conversely, changing a single entry of the tables can change the optimal value, so in the worst case every exact method reads all $m|\cA|$ entries. \hfill$\square$

\paragraph{The budget path.}
Let $\widehat V^*(B)$ be the optimal value of \eqref{eq:allocation} at budget $B$ for a fixed action set. It is nonincreasing because the feasible sets are nested. It is convex because a mixture $\theta\pi^0+(1-\theta)\pi^1$ of optimal policies at budgets $B_0,B_1$ is feasible at $\theta B_0+(1-\theta)B_1$, the constraints being linear, and has the mixed volume. Let $(\mu,\lambda)$ be optimal prices at $B$ and $\pi'$ an optimal policy at $B'$. Applying \eqref{eq:certificate} at $B$ to $\pi'$, whose regret and coverage terms are nonnegative and whose mean count is at most $B'$, gives $\widehat V^*(B')-\widehat V^*(B)\ge\lambda(B-B')$, so $-\lambda$ is a subgradient. For each $\lambda\ge0$, the minimizing rule at $(\mu(\lambda),\lambda)$, mixed to meet the coverage target, satisfies the optimality conditions of Proposition~\ref{prop:dual} at the budget equal to its own mean count. By the identity above its volume equals $D(\mu,\lambda)$, so by weak duality it is optimal at that budget and $(\mu,\lambda)$ are optimal prices there. That budget is nonincreasing in $\lambda$ because subgradients of a convex function are monotone, and at the finitely many values of $\lambda$ where it jumps, mixing tied actions fills in the budgets between. One sweep of $\lambda$ therefore traces $\widehat V^*$ between the smallest feasible budget and the budget at which $\lambda$ reaches zero.

\paragraph{The gain over fixed PCP, and the draws it saves.}
If fixed PCP meets the coverage target exactly and spends the budget, the last two terms of \eqref{eq:certificate} vanish, so $\widehat V(\text{fixed})-\widehat V^*(B)=\widehat R_F$, its priced regret; by Proposition~\ref{prop:dual} this is zero exactly when $(B,r)$ minimizes every $g_i$. The derivation of \eqref{eq:certificate} uses only linearity in the policy and the two coupling constraints, so the same identity holds for the population program \eqref{eq:population} over a finite menu, with $\mathcal P_X$ in place of the empirical input distribution and population tables, whenever optimal population prices exist, as under Slater's condition. For the draws saved, let $\lambda(B-\Delta)$ be an optimal price at budget $B-\Delta$. The subgradient inequality of the budget path at $B-\Delta$ gives
\[
 \begin{aligned}
 \widehat V^*(B-\Delta)
 &\le\widehat V^*(B)+\Delta\,\lambda(B-\Delta)\\
 &=\widehat V(\text{fixed})-\widehat R_F
   +\Delta\,\lambda(B-\Delta),
 \end{aligned}
\]
so the allocation at budget $B-\Delta$ is no larger than fixed PCP at budget $B$ whenever $\Delta\,\lambda(B-\Delta)\le\widehat R_F$. When prices vary little over this range, the saving is about $\widehat R_F/\lambda$ draws per prediction: the regret measured in units of the value of a draw.

\subsection{Proofs for Section~\ref{sec:theory}}
\label{app:theory}
Write $\omega_d$ for the volume of the unit ball in $\bR^d$. Throughout, the radius is constant and $r(x)\equiv1$, so the score of $k$ samples is $S=\min_{j\le k}\|Y-\widehat Y_j\|_2$ and a sample's ball of radius $r$ contains $y$ exactly when $\|y-\widehat Y_j\|_2\le r$.

\paragraph{Coverage under separated modes.}
Under Assumption~\ref{ass:modes}, if some sample lies in the ball containing $Y$, both lie in one ball of radius $\epsilon$, so $S\le2\epsilon$. Otherwise every sample lies in another ball, because the generator is supported on the union of the balls, and every point of another ball is at distance at least $2(R+1)\epsilon-2\epsilon=2R\epsilon$ from $Y$, so $S\ge2R\epsilon$. Given $X=x$, the response lies in ball $j$ with probability $p_j(x)$, and each independent sample misses that ball with probability $1-\widehat p_j(x)$. Hence, with $c_k(x)$ as in \eqref{eq:modecov},
\begin{equation}
 \bP(S\le r\mid X=x)=c_k(x)\ \ \text{for } r\in[2\epsilon,2R\epsilon),\qquad
 \bP(S\le r\mid X=x)\le c_k(x)\ \ \text{for } r<2R\epsilon,
 \label{eq:geometry}
\end{equation}
which proves \eqref{eq:modecov} and shows that $c_k(x)$ is nondecreasing in $k$.

\paragraph{Proof of Theorem~\ref{thm:threshold}.}
For a count rule $\kappa$, averaging \eqref{eq:geometry} over $X$ and $\kappa(X)$ shows that the population score distribution $F_\kappa$ equals $\bE[c_{\kappa(X)}(X)]$ on $[2\epsilon,2R\epsilon)$ and is at most this value below $2R\epsilon$. The calibrated radius $\inf\{r:F_\kappa(r)\ge1-\alpha\}$ is therefore at most $2\epsilon$ if $\bE[c_{\kappa(X)}(X)]\ge1-\alpha$, and at least $2R\epsilon$ otherwise. For uniform allocation, $\bE[c_B(X)]<1-\alpha$ by the condition, so its radius is at least $2R\epsilon$; every set contains a ball of that radius, and $V_F\ge\omega_d(2R\epsilon)^d$. By the condition, some count rule $\kappa$ with $\bE[\kappa(X)]\le B$ has $\bE[c_{\kappa(X)}(X)]\ge1-\alpha$, so its radius is at most $2\epsilon$ and each set is a union of $\kappa(x)$ balls of radius at most $2\epsilon$, with expected volume at most $\bE[\kappa(X)]\,\omega_d(2\epsilon)^d\le B\,\omega_d(2\epsilon)^d$. The optimal allocation is no larger, so $V_A\le B\,\omega_d(2\epsilon)^d$, and dividing the two bounds gives $V_A/V_F\le B/R^d$.

\emph{Finite calibration.} With $n$ i.i.d.\ calibration points and $\psi=\bE[c_{\kappa(X)}(X)]$, the number of calibration scores at most $2\epsilon$ is $\mathrm{Bin}(n,\psi)$ by \eqref{eq:geometry}, and the calibrated radius is at most $2\epsilon$ if this number is at least $\lceil(n+1)(1-\alpha)\rceil$ and at least $2R\epsilon$ otherwise. Hence $\bP\{\widehat r>2\epsilon\}=\bP\{\widehat r\ge2R\epsilon\}=\bP\{\mathrm{Bin}(n,\psi)<\lceil(n+1)(1-\alpha)\rceil\}$, which tends to zero when $\psi>1-\alpha$ and to one when $\psi<1-\alpha$. \hfill$\square$

\paragraph{Proof of Corollary~\ref{cor:conditional}.}
At a constant radius $r\in[2\epsilon,2R\epsilon)$, \eqref{eq:geometry} gives input $x$ with $k$ samples conditional coverage $c_k(x)$, which is nondecreasing in $k$. Input $x$ therefore attains conditional coverage $1-\alpha$ exactly when $\kappa(x)\ge k_\alpha(x)$. Uniform allocation attains it at almost every input exactly when $B\ge\sup_xk_\alpha(x)$, where suprema over inputs are essential suprema, and the smallest mean count of a deterministic rule attaining it is $\bE[k_\alpha(X)]$, attained by $\kappa=k_\alpha$; the difference is the saving. With a correct generator and $J(x)$ equally weighted modes, $c_k(x)=1-\{1-1/J(x)\}^k$, so $k_\alpha(x)=\max\{k_{\min},\lceil\log(1/\alpha)/\log\{J/(J-1)\}\rceil\}$ with $J=J(x)$ and $k_{\min}=\min\cK$, and $1/J\le\log\{J/(J-1)\}\le1/(J-1)$ gives $(J-1)\log(1/\alpha)\le k_\alpha(x)\le J\log(1/\alpha)+1$ whenever $k_\alpha(x)>k_{\min}$. \hfill$\square$

\paragraph{Proof of Proposition~\ref{prop:oracle}.}
If $k<k_\alpha(x)$, then $c_k(x)<1-\alpha$, and by \eqref{eq:geometry} every radius below $2R\epsilon$ gives conditional coverage below $1-\alpha$; the smallest radius attaining the target is at least $2R\epsilon$, and the set contains a ball of that radius. Under uniform allocation with $B$ samples, the inputs with $k_\alpha(X)>B$ therefore have sets of volume at least $\omega_d(2R\epsilon)^d$, so $V_F\ge\bP\{k_\alpha(X)>B\}\,\omega_d(2R\epsilon)^d$. If $k\ge k_\alpha(x)$, \eqref{eq:geometry} gives conditional coverage $c_k(x)\ge1-\alpha$ at radius $2\epsilon$, so the smallest such radius is at most $2\epsilon$ and the union of $k$ balls has volume at most $k\,\omega_d(2\epsilon)^d$. The allocation $\kappa=k_\alpha$ has mean count $\bE[k_\alpha(X)]\le B$ and expected size at most $\bE[k_\alpha(X)]\,\omega_d(2\epsilon)^d\le B\,\omega_d(2\epsilon)^d$, which bounds the optimal allocation; dividing gives the ratio. \hfill$\square$

\subsection{Oracle Radii in the Interval Model}
\label{app:efficiency}
The next result isolates a benefit that remains when every input already has the smallest radius attaining the target conditional coverage.
\begin{proposition}[Allocation with oracle radii]
\label{prop:sqrt}
Suppose $\mathcal P_x=\widehat f(x)$ is uniform on an interval of length $s(x)$, with $0<s_{\min}\le s(x)\le s_{\max}$. For each input and count, choose the smallest radius attaining conditional coverage $1-\alpha$. With inactive count caps, allocate counts by randomized rounding of
\[
 \kappa^*(x)=\frac{B\sqrt{s(x)}}{\bE\sqrt{s(X)}}
\]
to adjacent integers, preserving its conditional mean. As integer $B\to\infty$, the expected lengths under fixed counts and this allocation satisfy
\begin{equation*}
 V_F(B)-V_{\mathrm{sqrt}}(B)
 =\frac{A_\alpha}{B}\,\mathrm{Var}\{\sqrt{s(X)}\}+O(B^{-2}),
\end{equation*}
where $A_\alpha=\log(1/\alpha)-2+2\sqrt\alpha>0$.
\end{proposition}
Both rules have the same conditional coverage and average count. The leading gain is proportional to $\mathrm{Var}\{\sqrt{s(X)}\}$, quantifying the role of scale heterogeneity. Together with Theorem~\ref{thm:threshold}, it distinguishes large gains from crossing a representation threshold and smooth $1/B$ gains from reducing residual excess volume after local radius adaptation.

\paragraph{Proof of Proposition~\ref{prop:sqrt}.}
For the oracle radius at input $x$ and count $k$, we first establish the expansion
\[
 v_k(x;r_k(x))=s(x)\{1-\alpha+A_\alpha/k+O(k^{-2})\}.
\]
The remainder is uniform over the bounded scale family in the proposition. Minimizing the leading excess term under $\bE K(X)=B$ gives the square-root allocation.

Normalize the support to $[0,1]$. Integrating the probability that all $k$ centers miss a radius-$t$ neighborhood gives, for $0\le t\le1$,
\[
 1-c(k,t)=\frac{2(1-t)^{k+1}+(k-1)(1-2t)_+^{k+1}}{k+1}.
\]
The expected covered length inside the support equals $c(k,t)$ because the density is one, and each end contributes expected length $t-\{1-(1-t)^{k+1}\}/(k+1)$ outside it, from integrating $1-(1-u)^k$ over $u\in[0,t]$. Setting $t=z/k$ gives $1-c(k,z/k)=e^{-2z}+O(1/k)$ uniformly on compact $z$-intervals, with derivative in $z$ bounded away from zero near $z=\tfrac12\log(1/\alpha)$, so $t_k=\log(1/\alpha)/(2k)+O(k^{-2})$; substituting into the outside length gives $A_\alpha$, and $A_\alpha=2(z-1-\log z)>0$ with $z=\sqrt\alpha$. By the Cauchy--Schwarz inequality, $\bE\{s(X)/K(X)\}\,\bE K(X)\ge\{\bE\sqrt{s(X)}\}^2$, with equality for $K^*$. Bounded scales keep $K^*$ proportional to $B$, and randomized rounding preserves the mean count and changes $\bE\{1/K\mid X\}$ by $O(B^{-3})$. The leading excess is $A_\alpha\{\bE\sqrt s\}^2/B$ for $K^*$ against $A_\alpha\bE s/B$ for $K\equiv B$, and their difference is the claim. \hfill$\square$

\section{Experimental Details}
\label{app:experiments}

\subsection{Implementation}
\label{app:implementation}

\paragraph{Data and generators.}
The synthetic study uses a conditional quantile network with two hidden layers of 64 ReLU units, which outputs the quantiles at levels $0.01,\ldots,0.99$ and is trained with the pinball loss once per DGP on 2,000 observations; a sample interpolates its quantile curve at a uniform level. Each seed draws 500 observations for $\mathcal D_m$, 500 for $\mathcal D_n$, 1,000 for testing, and 4,000 for the conditional-coverage audit. Every method uses the same 1,000 labeled observations: methods without a learned component calibrate on $\mathcal D_m\cup\mathcal D_n$, and the others learn on $\mathcal D_m$ and calibrate on $\mathcal D_n$.

\paragraph{Conditional regression.}
We estimate $c_k$ and $v_k$ at a new input by nonparametric nearest-neighbor regression. After standardizing the features, $\widehat c_k(x)$ and $\widehat v_k(x)$ average the per-input estimates of the $60$ nearest points of $\mathcal D_m$, with the same weights for every count and radius, which preserves the monotone and concave shape of the estimated curves in $k$ (Proposition~\ref{prop:rb}). The features are the input coordinates for the synthetic studies and sample-free summaries of the generator for the real data. For a nearest-neighbor generator, these are the log mean distance of the neighbors' responses to their coordinatewise median, the fractions within $0.05$ and $0.2$ standardized units of it, and the log of the 90th percentile of those distances; for a mixture density network, they are the log total variance, the entropy and maximum of the mixture weights, and the log mean component scale. They are fixed functions of the input, so Theorem~\ref{thm:validity} applies unchanged. In a sensitivity analysis on the synthetic DGPs at $B=5$, using $30$ or $120$ neighbors instead of $60$, or averaging within a fixed partition of the input space, changes the mean set size by less than $1\%$.

\paragraph{Volume estimation.}
The estimator \eqref{eq:rb} needs only the observed response and the sample pool, not the generator density. In one dimension, $\widehat h_k(X_i;y)$ is constant between consecutive endpoints $\widehat Y_{i,j}\pm r$, so $\widehat v_k(X_i)$ is computed exactly by summing segment lengths weighted by $\widehat h_k$. In two dimensions, we draw integration points by choosing a pool center uniformly and then a point uniformly in its radius-$r$ disk. This proposal has density $H(X_i;y)/(N\pi r^2)$ on the union of the disks, so
\begin{equation*}
 \widehat v_k(X_i)
 =N\pi r^2\,\bE_U\left[\frac{\widehat h_k(X_i;U)}{H(X_i;U)}\,\middle|\,\widehat Y_{i,1:N}\right],
\end{equation*}
estimated for all counts from one set of 256 points. Test-set areas are computed exactly by disk-union boundary integration.

\paragraph{Sampling, optimization, and budget.}
Each input in $\mathcal D_m$ has a pool of $N=100$ samples, the candidate counts are $\{2,\ldots,2B-2\}$, and the radius grid has 80 geometrically spaced values chosen from $\mathcal D_m$ and its pool. The allocation is computed by the exact dual of Theorem~\ref{thm:dual} (Appendix~\ref{app:dual}). Calibration and test inputs use fresh samples. The budget caps the mean number of prediction-time generator samples per input, including auxiliary samples; the offline pool on $\mathcal D_m$ is not counted in $B$. Because the allocation enforces the budget as an average over $\mathcal D_m$, we report the realized mean number of samples on test inputs.

\paragraph{Hardware.}
Experiments ran on three machines: a MacBook Pro with an Apple M2 Pro chip (12 CPU cores, 32\,GB memory), a MacBook Air with an Apple M4 chip (10 CPU cores, 16\,GB memory), and a desktop workstation with an NVIDIA GeForce RTX 5070 Ti GPU. The synthetic and real-data sweeps run on CPU and were divided across the three machines by random seed; GPU computation, such as training the mixture density network for the taxi data and sampling answers from the vision-language model, used the workstation.

\paragraph{Evaluation metrics.}
Marginal coverage and mean set size are computed on the test observations, and cost is the realized mean number of generator samples per test input. For conditional coverage, let $\mathcal B_j$ be input cells, with audit count $n_j$, audit coverage $\widehat c_j$, and weight $\omega_j=n_j/n_{\mathrm{audit}}$. We report
\[
 \mathrm{MACE}=\sum_j\omega_j|\widehat c_j-(1-\alpha)|,
\]
which averages over cells and is distinct from pointwise conditional coverage. The cells are input-only recursive median partitions fitted on $\mathcal D_m$, with 10 cells for scalar inputs and 16 for two-dimensional inputs. Intervals are 95\% intervals across seeds.

\subsection{Compared Methods}
\label{app:planned}
The PCP family differs only in which decisions vary (Section~\ref{sec:formulation}): fixed PCP uses $K\equiv B$ and $r(x)\equiv1$; \casa{} lets the count vary with one shared radius: it solves the allocation at each radius of the grid $\cR$ and keeps the radius whose solution has the smallest estimated mean set size on $\mathcal D_m$. \casa{} (joint) lets each input choose both its count and its radius from $\cR$, by the same exact dual over count--radius pairs (Appendix~\ref{app:dual}), with the budget applying to the counts. Input-independent allocation draws every count from one distribution that does not depend on the input, under the same budget cap and with a shared radius; in the ablation of Section~\ref{sec:experiments}, the counts are uniform on $\{2,\ldots,8\}$ with one calibrated radius. In that ablation, size is the geometric mean over DGPs of the ratio to fixed PCP, and coverage and MACE are means over DGPs.

The external constructions follow their published definitions and use the same generator and the conformal rank $\lceil(m+1)(1-\alpha)\rceil$, with an unbounded set when the rank exceeds $m$. Let $\widehat Y_1,\ldots,\widehat Y_K\sim G_x$ be the centers at input $x$ and $V(x,y)=\min_k\|y-\widehat Y_k\|$ the PCP score.
\begin{itemize}
\item \textbf{CP$^2$} \citep{plassier2025cp2}, linear adjustment. An independent auxiliary sample $\widetilde Y_1,\ldots,\widetilde Y_A\sim G_x$ defines $\widehat\sigma(x)$ as the $\lceil A(1-\alpha)\rceil$-th smallest of $V(x,\widetilde Y_j)$, $j\le A$. The score is $V(x,y)/\widehat\sigma(x)$, and the set is $\bigcup_k B(\widehat Y_k,\widehat q\,\widehat\sigma(x))$. Because $\widehat\sigma$ is recomputed from the realized centers, it adapts to any count. The auxiliary size $A=\lfloor B/2\rfloor$ is the same at every input and counted toward the budget, so fixed CP$^2$ uses $B-A$ centers and its count-allocated version satisfies $\mathbb E[K(X)]\le B-A$.
\item \textbf{RCP} \citep{plassier2025rcp}, linear adjustment as in its PCP experiments. A quantile regression fitted on $\mathcal D_m$ estimates $\widehat\sigma(x)$, the conditional $(1-\alpha)$-quantile of $V(X,Y)$ given $X=x$; the score and set are as for CP$^2$. Since the nearest-center distance decreases with the count, the count-allocated version estimates $\widehat\sigma(x,k)$ with the count as an additional regressor.
\item \textbf{CP4Gen} \citep{yang2026cp4gen}. $k$-means clustering of the centers into $L$ groups (solved exactly by dynamic programming for scalar responses) gives weights $w_\ell$ (cluster proportions), means $\mu_\ell$, and sample covariances $\Sigma_\ell$ (zero for singletons). The score is $-\log\max_\ell w_\ell\,\mathcal N(y;\mu_\ell,\Sigma_\ell+\beta^2 I)$ with $\beta^2=10^{-6}$, and the set is a union of ellipsoids. The number of clusters is chosen on $\mathcal D_m$ at $K=B$ to minimize mean volume over the grid of the official implementation, whose largest value $L=B$ makes every sample its own cluster. In the count-allocated version an input with $k$ samples uses $\min(L,k)$ clusters, and the choice $L=B$ carries over as one cluster per sample.
\item \textbf{C-PCP} \citep{dheur2025unified}. The score is the empirical distribution function, over $A$ auxiliary draws, of the nearest-center distance, $\widehat F_x(V(x,y))=A^{-1}\sum_j\one\{V(x,\widetilde Y_j)\le V(x,y)\}$, and the set is the union of balls whose radius is the corresponding auxiliary order statistic. The score takes only $A+1$ values, so with $A=\lfloor B/2\rfloor$ at the budgets studied here the $90\%$ threshold reaches one and every set is the whole space. We therefore give C-PCP $A=20$ auxiliary draws beyond the budget, identical for its fixed and count-allocated versions, so that the comparison isolates the allocation of centers.
\end{itemize}
The Rao--Blackwell tables of Proposition~\ref{prop:rb} assume balls with a common radius, so for the external constructions the coverage and volume of each count are estimated by direct Monte Carlo over collections of pool samples; the allocation and calibration steps are unchanged.

\paragraph{Budget curves.}
Figure~\ref{fig:baselines_rare} compares each construction with and without \casa{} on rare modes over budgets $B=3,\ldots,10$. Budgets are matched within each construction; CP$^2$ spends $\lfloor B/2\rfloor$ of its budget on auxiliary draws and C-PCP uses 20 auxiliary draws beyond it, so the horizontal axis counts realized calls, auxiliary draws included. Marginal coverage stays within $[89.6,90.5]\%$ for every construction except C-PCP ($91.5$--$93.2\%$), whose discrete score is conservative.

\begin{figure*}[t]
\centering
\includegraphics[width=\textwidth]{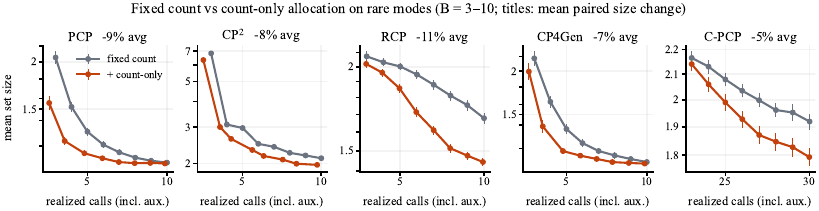}
\caption{Fixed count versus \casa{} on rare modes for $B=3,\ldots,10$, with 95\% intervals over 100 seeds. The horizontal axis counts realized samples, auxiliary draws included; panel titles give the mean size change.}
\label{fig:baselines_rare}
\end{figure*}

\subsection{Synthetic DGPs}
\label{app:synthetic}
\paragraph{DGPs.}
Each DGP draws $X\sim\mathrm{Unif}(0,1)$ and $Y=0.3\sin(2\pi X)+Z$, where the noise $Z$ changes with the input as follows, with $\varepsilon\sim\mathcal N(0,1)$, $U\sim\mathrm{Unif}(-1,1)$, $M$ uniform on $\{-1.5,-0.5,0.5,1.5\}$, $T\sim t_5$, the logistic function $\ell(t)=1/(1+e^{-t})$, and all variables independent.
\begin{center}
\small
\setlength{\tabcolsep}{8pt}
\begin{tabular}{ll}\toprule
DGP & Noise $Z$ \\
\midrule
Compact to extended & $a(X)\,U+0.02\,\varepsilon$, \quad $a(X)=0.08$ if $X<0.5$, else $1$ \\
Uniform bulk & $a(X)\,U+0.02\,\varepsilon$, \quad $a(X)=0.25$ if $X<0.5$, else $1$ \\
Gaussian scale & $s(X)\,\varepsilon$, \quad $s(X)=0.12$ if $X<0.5$, else $0.8$ \\
Sharp scale & $s(X)\,\varepsilon$, \quad $s(X)=0.06$ if $X<0.75$, else $0.45$ \\
Emerging modes & $\{0.02+0.98\,\ell(16(X-0.5))\}\,M+0.04\,\varepsilon$ \\
Separating modes & $\pm\{0.02+0.98\,\ell(12(X-0.5))\}+0.06\,\varepsilon$, each sign equally likely \\
Rare modes & $\one\{X\ge0.8\}\,M+0.06\,\varepsilon$ \\
Mode width & $M+s(X)\,\varepsilon$, \quad $s(X)=0.025$ if $X<0.5$, else $0.15$ \\
Heavy tail & $0.35\,T$ with probability $0.02+0.28\,\ell(12(X-0.6))$, else $0.06\,\varepsilon$ \\
\bottomrule\end{tabular}
\end{center}
The learner observes only $X$ and the labeled data.

\paragraph{Protocol.}
We use budgets $B=3,\ldots,15$ with 100 seeds each and target coverage $90\%$. Each input in $\mathcal D_m$ has a pool of $N=100$ samples, and both variants use the same grid $\cR$ of 80 radius values.

\paragraph{Uncertainty-set geometry.}
Figure~\ref{fig:synthetic_sets} visualizes the calibrated uncertainty sets of fixed PCP, \casa{}, and \casa{} (joint) in one replication on two scalar and two vector DGPs. These individual realizations illustrate the geometry of the sets; aggregate comparisons are in Figure~\ref{fig:frontiers} and Table~\ref{tab:baselines}.

\begin{figure}[p]
\centering
\includegraphics[width=\linewidth]{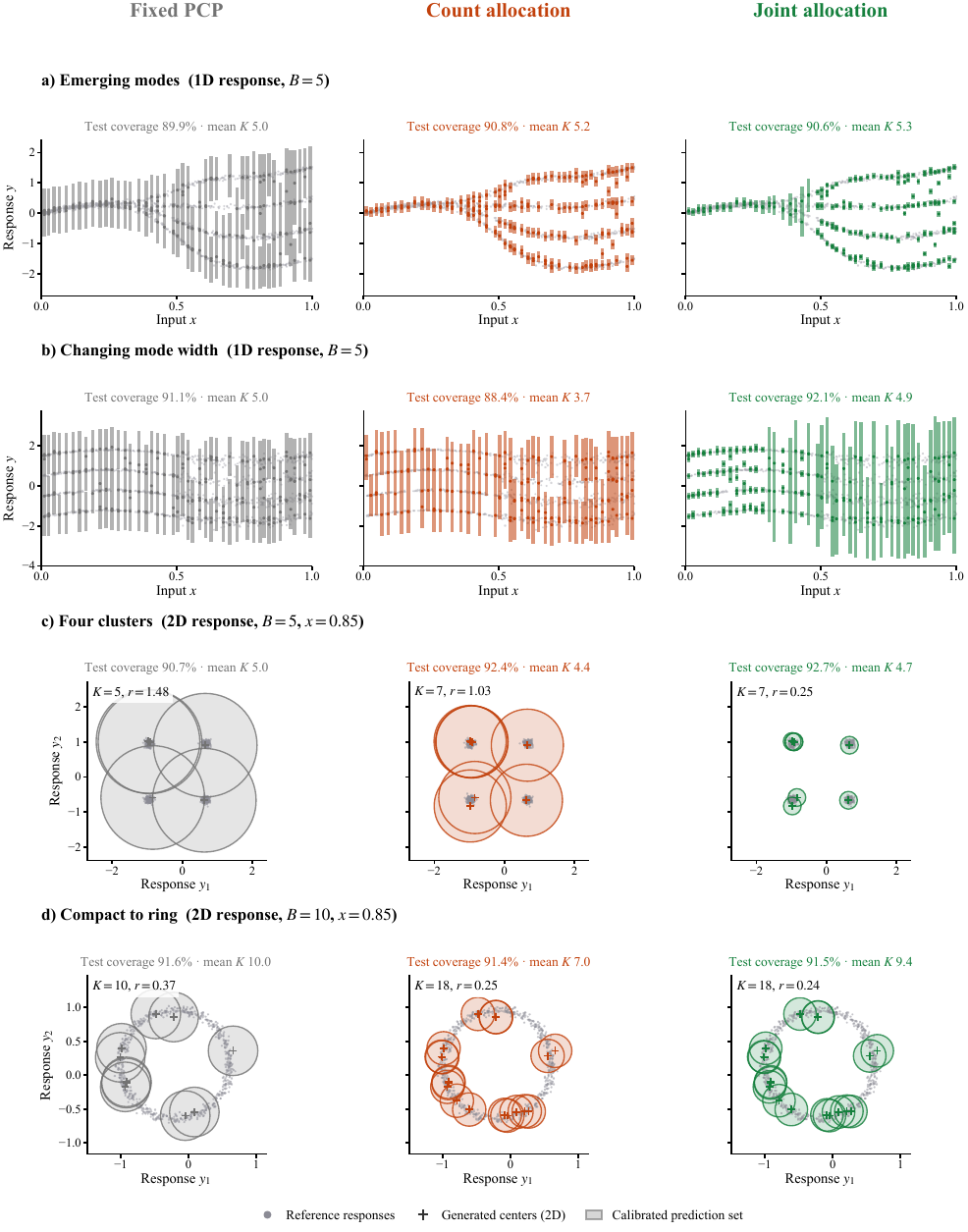}
\caption{Calibrated uncertainty sets on four synthetic DGPs for fixed PCP, \casa{}, and \casa{} (joint). (a--b)~Gray points are test pairs, and colored segments are the sets. (c--d)~Gray points are responses at one input, crosses are samples, and circles outline the sets. Headings give marginal coverage and mean samples.}
\label{fig:synthetic_sets}
\end{figure}

\subsection{Taxi Routes}
\label{app:porto}
\paragraph{Data and task.}
The Porto taxi data \citep{porto2015trajectory}, released for the ECML/PKDD 2015 trajectory-prediction challenge, record 1.7 million trips with one GPS fix every 15\,s. We use the first 300,000 trips, drop trips with missing data or a jump of more than 500\,m between fixes, and keep windows in an $8\times6$\,km box around the city centre in which the taxi is moving. From a taxi's current position and its displacements over the last 15\,s and 60\,s, we predict its displacement to the point 300\,m further along its route. Measuring the horizon in distance rather than time removes speed and leaves route choice: on a straight road the answer is nearly a single point, while at a junction it splits into one point per exit. Splits are by trip, and each evaluation trip contributes one random window: 2,000 trips for $\mathcal D_m$, 2,000 for $\mathcal D_n$, 3,000 for testing, and 3,000 for the conditional-coverage audit; the remaining trips supply two million generator windows. Set size is area in square metres.

\paragraph{Generator.}
The main generator is a trajectory library: a sample is the 300\,m displacement of one of the 20 training windows nearest in position and recent motion, with Gaussian jitter of about 4\,m. As a learned alternative, a mixture density network maps random Fourier features of the position and the recent displacements to a 20-component Gaussian mixture, trained by maximum likelihood on each split's generator windows.

\paragraph{Results.}
Over ten splits, \casa{} keeps coverage at $90.3$--$90.4\%$ (fixed PCP $90.1$--$90.2\%$) and reduces mean set area to 0.68, 0.63, and 0.88 of fixed PCP's at budgets 3, 5, and 10, and to 0.60 and 0.80 at equal realized samples at budgets 5 and 10 (Table~\ref{tab:porto}); \casa{} (joint) gives 0.58, 0.56, and 0.85. Taxis on through-roads receive two samples and taxis near junctions up to eight (Figure~\ref{fig:porto}). The extra samples cover exits that fixed PCP misses and lower the shared calibrated radius, from 63 to 44\,m on one split, so sets shrink at every sample level. The gain holds with 10 or 40 neighbours, for route distances of 200 and 500\,m, under a time split in which the generator sees only earlier trips, and with the mixture density network.

\begin{table*}[t]
\centering\small
\caption{Taxi routes, all configurations: mean set area relative to fixed PCP over ten splits, at the same budget and at equal realized samples. Bold: paired $t$-test $p<0.05$. Last column: \casa{}'s mean samples at $B=3,5,10$.}
\label{tab:porto}
\setlength{\tabcolsep}{3.2pt}
\begin{tabular}{lccccccccc}\toprule
& \multicolumn{3}{c}{\casa{}, same cap} & \multicolumn{2}{c}{Equal samples} & \multicolumn{3}{c}{\casa{} (joint), same cap} & \\
\cmidrule(lr){2-4}\cmidrule(lr){5-6}\cmidrule(lr){7-9}
Configuration & $B{=}3$ & $5$ & $10$ & $5$ & $10$ & $3$ & $5$ & $10$ & Mean $K$ \\
\midrule
20-NN, 300\,m (main) & \textbf{0.68} & \textbf{0.63} & \textbf{0.88} & \textbf{0.60} & \textbf{0.80} & \textbf{0.58} & \textbf{0.56} & \textbf{0.85} & 3.0/4.9/9.7 \\
10-NN, 300\,m & \textbf{0.69} & \textbf{0.63} & \textbf{0.90} & \textbf{0.60} & \textbf{0.78} & \textbf{0.56} & \textbf{0.52} & 0.91 & 3.0/4.9/9.5 \\
40-NN, 300\,m & \textbf{0.67} & \textbf{0.62} & \textbf{0.86} & \textbf{0.57} & \textbf{0.77} & \textbf{0.54} & \textbf{0.55} & \textbf{0.85} & 3.0/4.8/9.7 \\
20-NN, 300\,m, time split & \textbf{0.67} & \textbf{0.64} & \textbf{0.90} & \textbf{0.60} & \textbf{0.84} & \textbf{0.51} & \textbf{0.54} & \textbf{0.85} & 2.9/4.9/9.5 \\
20-NN, 200\,m & \textbf{0.62} & \textbf{0.73} & 0.94 & \textbf{0.66} & \textbf{0.84} & \textbf{0.52} & \textbf{0.66} & \textbf{0.89} & 2.9/4.8/8.7 \\
20-NN, 500\,m & \textbf{0.90} & \textbf{0.66} & \textbf{0.76} & \textbf{0.66} & \textbf{0.73} & \textbf{0.73} & \textbf{0.59} & \textbf{0.77} & 2.9/5.0/9.9 \\
MDN, 300\,m & \textbf{0.76} & \textbf{0.70} & 1.00 & \textbf{0.69} & 0.94 & \textbf{0.59} & \textbf{0.61} & \textbf{0.93} & 3.0/5.0/9.6 \\
\bottomrule\end{tabular}
\end{table*}

\subsection{Protein Backbone Angles}
\label{app:protein}
\paragraph{Data and task.}
We predict a residue's backbone dihedral angles $(\phi,\psi)$, its position on the Ramachandran plot, from the sequence of its chain. We use the CATH 4.2 chains of \citet{ingraham2019generative}, split by CATH topology so that no test chain has a homologous chain in training, and compute the angles from the backbone coordinates of N, C$_\alpha$, and C, with $\psi$ shifted to $[-110^\circ,250^\circ)$ so that the helical and extended regions are contiguous. Evaluation residues are drawn at random from the test chains, 3,000 each for $\mathcal D_m$, $\mathcal D_n$, testing, and the conditional-coverage audit. Set size is area in square degrees.

\paragraph{Generator.}
Each residue is represented by the last-layer embedding of ESM-2 with 35 million parameters \citep{lin2023esm}, which sees the whole chain, projected onto 64 principal components. A library holds the 1.3 million residues of 6,000 random training chains, and a sample is the $(\phi,\psi)$ of one of the 20 library residues closest in embedding.

\paragraph{Results.}
Over ten splits, \casa{} reduces mean set area to 0.76, 0.87, and 0.94 of fixed PCP's at budgets 3, 5, and 10, and to 0.76 and 0.86 at equal realized samples (Table~\ref{tab:protein}), and lowers MACE from 2.85 to 2.58 points at $B=5$. It assigns 3.8 samples on average to residues in the helical region, 4.8 to the extended region, and 5.8 to residues with $\phi>0$, against five for fixed PCP (Figure~\ref{fig:protein}). The gain holds for 10 and 40 neighbours and for the ESM-2 model with 150 million parameters.

\begin{figure*}[t]
\centering
\includegraphics[width=\textwidth]{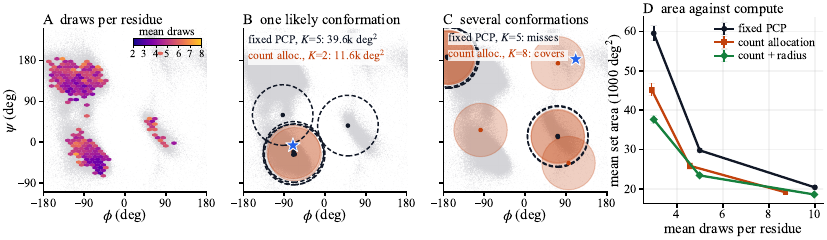}
\caption{Protein backbone angles at $B=5$. (A)~Mean samples \casa{} assigns, by position on the Ramachandran plane. (B)~A helical residue: fewer samples give a smaller set. (C)~A rare conformation that fixed PCP's samples miss. (D)~Set area against mean samples over ten splits.}
\label{fig:protein}
\end{figure*}

\begin{table*}[t]
\centering\small
\caption{Protein backbone angles, all configurations: mean set area relative to fixed PCP over ten splits, at the same budget and at equal realized samples. Bold: paired $t$-test $p<0.05$. Last column: \casa{}'s mean samples at $B=3,5,10$.}
\label{tab:protein}
\setlength{\tabcolsep}{3.2pt}
\begin{tabular}{lccccccccc}\toprule
& \multicolumn{3}{c}{\casa{}, same cap} & \multicolumn{2}{c}{Equal samples} & \multicolumn{3}{c}{\casa{} (joint), same cap} & \\
\cmidrule(lr){2-4}\cmidrule(lr){5-6}\cmidrule(lr){7-9}
Configuration & $B{=}3$ & $5$ & $10$ & $5$ & $10$ & $3$ & $5$ & $10$ & Mean $K$ \\
\midrule
20-NN, ESM-2 35M (main) & \textbf{0.76} & \textbf{0.87} & \textbf{0.94} & \textbf{0.76} & \textbf{0.86} & \textbf{0.63} & \textbf{0.79} & \textbf{0.91} & 2.9/4.6/8.7 \\
10-NN & \textbf{0.71} & \textbf{0.88} & \textbf{0.97} & \textbf{0.80} & \textbf{0.89} & \textbf{0.62} & \textbf{0.81} & \textbf{0.94} & 3.0/4.7/8.7 \\
40-NN & \textbf{0.81} & \textbf{0.86} & \textbf{0.94} & \textbf{0.80} & \textbf{0.87} & \textbf{0.69} & \textbf{0.81} & \textbf{0.90} & 2.9/4.7/8.8 \\
20-NN, ESM-2 150M & \textbf{0.80} & \textbf{0.89} & \textbf{0.92} & \textbf{0.79} & \textbf{0.88} & \textbf{0.68} & \textbf{0.81} & \textbf{0.92} & 2.9/4.5/9.2 \\
\bottomrule\end{tabular}
\end{table*}

\begin{figure}[H]
\centering
\includegraphics[width=\textwidth]{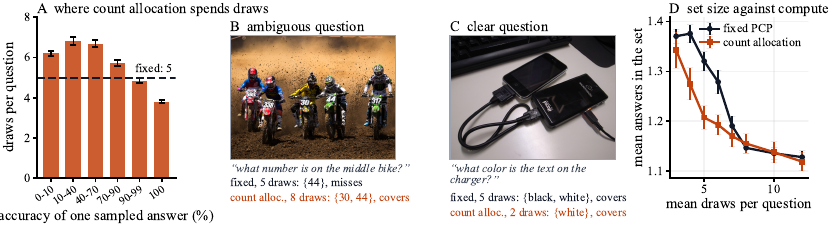}
\caption{TextVQA with Qwen2.5-VL-3B at an $80\%$ target and $B=5$. (A)~Mean samples \casa{} assigns, by question accuracy. (B)~An ambiguous question: \casa{}'s extra samples find the answer. (C)~A clear question: two samples suffice. (D)~Answers per set against mean samples over ten splits.}
\label{fig:textvqa}
\end{figure}

\subsection{Answer Sets for a Vision-Language Model}
\label{app:answers}
\paragraph{Data and task.}
TextVQA \citep{singh2019textvqa} asks about text in photographs, and OK-VQA \citep{marino2019okvqa} requires outside knowledge; both have ten human answers per question, and an answer counts as correct if at least three annotators gave it (TextVQA) or at least one did (OK-VQA). Each task is split at random into 1,500 questions each for $\mathcal D_m$, $\mathcal D_n$, and testing, with the rest for the conditional-coverage audit. The model limits which targets are reachable: an acceptable answer appears among 24 samples for about $88\%$ of TextVQA and $85\%$ of OK-VQA questions, so we target $80\%$ and $75\%$ coverage.

\paragraph{Model and sets.}
We use Qwen2.5-VL-3B \citep{bai2025qwen25vl}, prompted to answer with a single word or phrase, drawing 24 answers per question at temperature $0.7$. A ball around a discrete answer contains only that answer, so we use the sampling-only construction of \citet{su2024api}: draw $K$ answers and return every distinct answer whose frequency among them is at least a calibrated threshold. The allocation is unchanged, with exact coverage and size tables from the hypergeometric count of each answer in a random $K$-subset of the pool, the analogue of \eqref{eq:rb}. The regression features summarize one greedy decode: its token log-probabilities, entropies, and length.

\paragraph{Results.}
On TextVQA, \casa{} matches fixed sampling's set size with $44\%$, $29\%$, and $28\%$ fewer samples at $B=3$, $4$, and $5$, giving the questions the model finds hardest about 1.5 times the samples of those it always answers correctly (Figure~\ref{fig:textvqa}); at $B=3$, fixed sampling misses the target in three of ten splits, while \casa{} reaches it in all. On OK-VQA, sets are $11$--$15\%$ smaller at $B=5$--$8$.

\end{document}